\documentclass{article}

\usepackage[main, final]{neurips_2026}

\usepackage[utf8]{inputenc} 
\usepackage[T1]{fontenc}    
\usepackage{hyperref}       
\usepackage{url}            
\usepackage{booktabs}       
\usepackage{amsfonts}       
\usepackage{nicefrac}       
\usepackage{microtype}      
\usepackage{xcolor}         
\usepackage{algorithm}
\usepackage{algorithmic}
\usepackage{mathtools}

\hypersetup{colorlinks,allcolors=black}
\usepackage[table]{xcolor}
\usepackage{multirow} 
\usepackage{graphicx}
\usepackage{amsmath}
\usepackage{pifont}
\usepackage{bbm}
\usepackage{booktabs}
\usepackage{makecell}
\usepackage{amsthm}
\usepackage{cancel}
\usepackage{soul}
\usepackage{listings}
\usepackage{wrapfig}
\usepackage{subcaption}
\usepackage{amssymb}

\usepackage{etoc}

\definecolor{tablecell1}{RGB}{242, 242, 242}
\definecolor{tablecell2}{RGB}{205, 232, 248}
\definecolor{tablecell3}{RGB}{255, 242, 204}

\sethlcolor{tablecell2}

\definecolor{dkgreen}{rgb}{0,0.6,0}
\definecolor{gray}{rgb}{0.5,0.5,0.5}
\definecolor{mauve}{rgb}{0.58,0,0.82}
\definecolor{navy}{rgb}{0.0,0.0,0.5}

\definecolor{graybg}{gray}{0.95}       
\definecolor{bluebg}{RGB}{205, 232, 248} 

\theoremstyle{plain}
\newtheorem{theorem}{Theorem}[section]
\newtheorem{proposition}[theorem]{Proposition}
\newtheorem{lemma}[theorem]{Lemma}
\newtheorem{corollary}[theorem]{Corollary}
\theoremstyle{definition}

\theoremstyle{remark}

\title{SLOPE: Optimistic Potential Landscape Shaping for Model-based Reinforcement Learning}

\author{%
  \textbf{Yao-Hui Li}$^{1}$\thanks{Equal contribution.} \hspace{0.8em}
  \textbf{Zeyu Wang}$^{1}$\footnotemark[1] \hspace{0.8em}
  \textbf{Xin Li}$^{1}$\thanks{Correspondence to: Xin Li <xinli@bit.edu.cn>} \hspace{0.8em}
  \textbf{Wei Pang}$^{2}$ \hspace{0.8em}
  \textbf{Yingfang Yuan}$^{2,3}$ \\[0.2em]
  \textbf{Zhengkun Chen}$^{1}$ \hspace{0.8em}
  \textbf{Boya Zhang}$^{4}$ \hspace{0.8em}
  \textbf{Riashat Islam}$^{5}$ \hspace{0.8em}
  \textbf{Alex Lamb}$^{6}$ \hspace{0.8em}
  \textbf{Yonggang Zhang}$^{7}$ \\[0.7em]
  {\normalfont
  $^1$Beijing Institute of Technology, Beijing, China \quad
  $^2$Heriot-Watt University, Edinburgh, U.K.
  } \\[0.1em]
  {\normalfont
  $^3$Northumbria University, Newcastle upon Tyne, U.K.
  } \\[0.1em]
  {\normalfont
  $^4$Shenzhen Institutes of Advanced Technology, CAS, Shenzhen, China
  } \\[0.1em]
  {\normalfont
  $^5$Microsoft Research NY, USA \quad
  $^6$Tsinghua University, Beijing, China
  } \\[0.1em]
  {\normalfont
  $^7$Key Lab. of Symbolic Computation and Knowledge Engineering of MOE, Jilin University, China
  }
}

\begin{document}

\maketitle

\begin{abstract}
  Model-based reinforcement learning (MBRL) is sample-efficient but struggles in sparse reward settings. A critical bottleneck arises from the lack of informative gradients in sparse settings, where standard reward models often yield flat landscapes that struggle to guide planning. To address this challenge, we propose \underline{\textbf{S}}haping \underline{\textbf{L}}andscapes with \underline{\textbf{O}}ptimistic \underline{\textbf{P}}otential \underline{\textbf{E}}stimates (\textbf{SLOPE}), a novel framework that shifts reward modeling from predicting sparse scalars to constructing informative potential landscapes. SLOPE employs optimistic distributional regression to estimate high-confidence upper bounds, which amplifies rare success signals and ensures sufficient exploration gradients. Evaluations on \textbf{30+ tasks across 5 benchmarks} and \textbf{real-world robotic deployments}, demonstrate that SLOPE consistently outperforms leading baselines in \textbf{fully sparse}, \textbf{semi-sparse}, and \textbf{dense} rewards. The project code is available at \href{https://github.com/Arya87/SLOPE}{\textcolor{blue}{\nolinkurl{https://github.com/Arya87/SLOPE}}}.
\end{abstract}


\section{Introduction}

Model-based reinforcement learning (MBRL) uses learned dynamics and reward models to simulate future trajectories and evaluate policies, relying on dense reward signals to provide fine-grained feedback for planning \cite{tdmpc, tdmpc2, bmpc}. In MBRL, such rewards are typically approximated by a learned reward model to guide policy optimization. In practice, such dense and well-shaped rewards are rarely available, while sparse binary signals indicating task success or failure are much more common in real-world environments. This mismatch limits the ability of MBRL to exploit its planning capabilities, as sparse rewards provide insufficient guidance for policy improvement. Therefore, addressing this limitation is critical for fully realizing the potential of MBRL in real-world.

\begin{wrapfigure}{r}{0.5\textwidth}
    \vspace{-15pt}
    \centering
    \includegraphics[width=\linewidth]{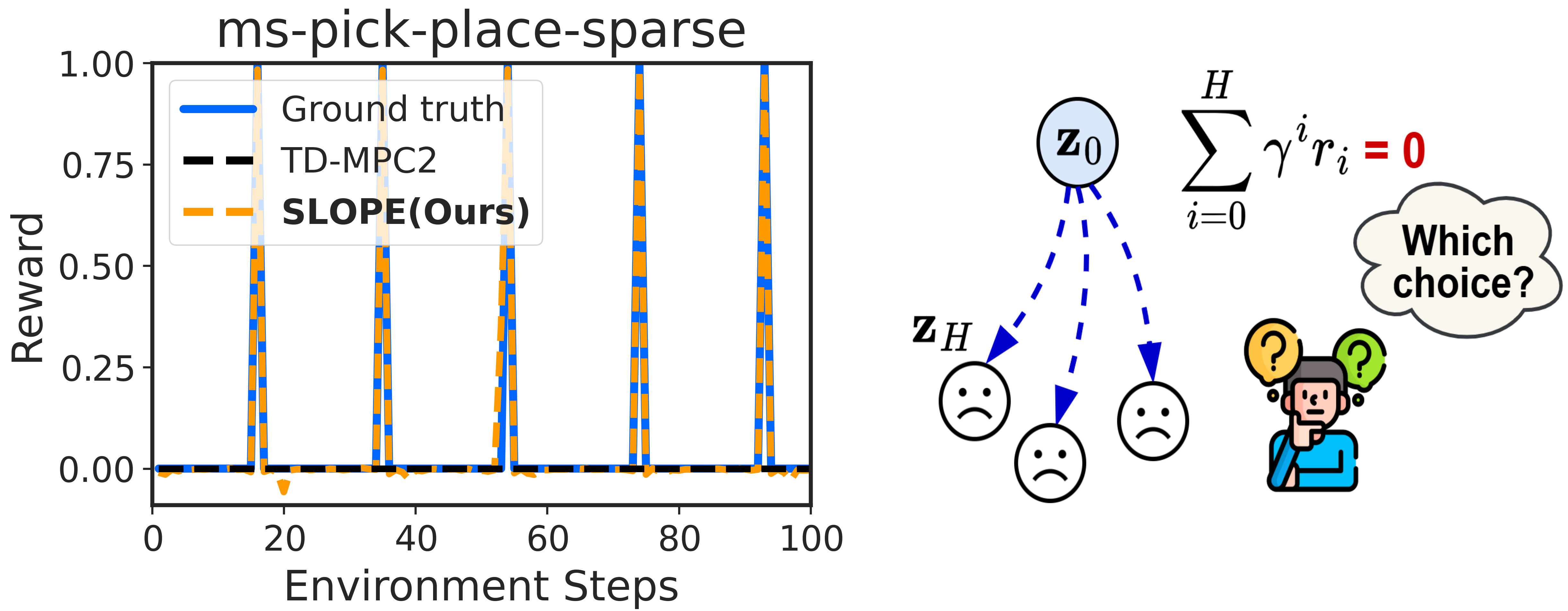}
    \caption{Key challenges in sparse-reward MBRL.}
    \label{fig:motivation}
    \vspace{-10pt}
\end{wrapfigure}

As illustrated in Fig.\ref{fig:motivation}, developing reward models in MBRL under sparse rewards presents two key challenges: i) \textbf{Reward learning failure under extreme data imbalance.} In sparse reward settings, the vast majority of trajectories yield zero rewards. This scarcity of successful samples makes it difficult for the reward model to extract meaningful success patterns, often causing the model to collapse toward the majority class (zero) \cite{dreamsmooth}. ii) \textbf{Ineffectiveness of ground-truth rewards for planning.} Most MBRL approaches \cite{tdmpc, tdmpc2, dreamerv3} focus on accurately regressing ground-truth scalar rewards. However, in sparse settings, strictly fitting these values results in a reward model that outputs zeros almost everywhere, even with perfect accuracy. This creates a gradient-free landscape during rollouts, offering the planner no intermediate signals to guide the agent toward the goal.

To address the aforementioned challenges, recent studies have sought to incorporate expert demonstrations to mitigate the exploration dilemma in sparse reward tasks. MoDem series \cite{modem, modemv2} introduced a multi-phase training framework that effectively leverages demonstration data to accelerate the training of visual MBRL algorithms. However, this approach focuses solely on enhancing the training strategy and still adheres to regressing ground-truth sparse rewards. Meanwhile, inverse RL \cite{adverserial_IRL, graph_IRL}, on the other hand, infers the reward function from expert demonstrations. However, in complex tasks, such approaches typically require a large amount of demonstration data to sufficiently cover the state-action space. Reward shaping techniques that incorporate intrinsic motivation, such as maximum entropy-based \cite{entropy_based_1, entropy_based_2}, count-based \cite{count_based_1, count_based_2}, or similarity-based \cite{similarity_based_1, similarity_based_2} approaches, can guide agents to discover meaningful behaviors in sparse reward environments. DEMO$^{3}$ \cite{demo3} artificially decomposed the sparse reward into multi-stage semi-sparse reward, assigning higher reshaped rewards to states in the later stages. Nevertheless, heuristic shaping relies on delicate, task-specific engineering and often misaligns with the true objective. This misalignment makes agents prone to \textit{reward hacking}, severely compromising the reliability of trajectory evaluation in sampling-based planners. Motivated by these limitations, a fundamental question arises.

\textit{How can we construct a dense, informative guiding signal for MBRL in sparse reward settings that provides reliable directional feedback for trajectory planning, while remaining aligned with the true task objective?}

In this paper, we propose \textbf{SLOPE} (\textbf{S}haping \textbf{L}andscapes with \textbf{O}ptimistic \textbf{P}otential \textbf{E}stimates), a novel framework that integrates Potential-Based Reward Shaping (PBRS) into MBRL to tackle the fundamental challenge of gradient starvation in sparse-reward environments. Unlike standard MBRS algorithms that regress sparse scalar rewards, SLOPE constructs a dense internal potential landscape to provide continuous directional guidance for planning. Specifically, it treats the agent’s learned $Q$-value function as a potential surface over the state space. This mechanism transforms the originally flat, gradient-free reward structure into a dense landscape, encouraging the agent to explore states with higher estimated potentials and providing directional reward signals during rollouts.

Although PBRS offers a dense potential construction, the effectiveness of the constructed landscape depends on the accuracy of the potential estimates. In sparse reward tasks, infrequent and delayed feedback often leads to underestimation of $Q$-values and slow convergence \cite{her}, resulting in a flat landscape that fails to facilitate exploration. To address this limitation, SLOPE employs optimism-driven landscape shaping. Unlike conventional methods that approximate the expected $Q$-value (mean), this mechanism leverages distributional value modeling to estimate a high-confidence upper bound. By amplifying the impact of rare but informative high-reward trajectories, the proposed mechanism ensures the shaped landscape possesses sufficient gradients to accelerate policy learning.

The versatility and robustness of SLOPE are underscored by an extensive empirical evaluation across \textbf{5} major benchmarks, encompassing over \textbf{30} diverse manipulation and motion tasks. By integrating SLOPE with diverse MBRL backbones (e.g., TD-MPC2 \cite{tdmpc2}, Dreamerv3 \cite{dreamerv3}) and evaluating across \textbf{varying reward densities} (sparse, semi-sparse, and dense reward), we demonstrate its consistent superiority over traditional scalar regression baselines. Beyond simulated environments, we further validate SLOPE through \textbf{real-world robotic} experiments, where the challenges of sparse rewards are inherently prevalent and critical. This comprehensive cross-backbone validation and real-world deployment establish SLOPE as a powerful, general-purpose enhancement for modern MBRLs.



\section{Related Work}
\label{related_work}
\subsection{MBRL with Sparse-Reward}
MBRLs improve sample efficiency by leveraging learned environment dynamics for planning. However, sparse reward settings introduce significant challenges, primarily due to imbalanced training data and the lack of informative gradients during rollouts. Recent approaches like MoDem \cite{modem} and MoDem-v2 \cite{modemv2} mitigate the exploration dilemma by leveraging demonstration data through a multi-stage training framework. While effective in pre-training, these methods generally adhere to the standard paradigm of regressing ground-truth scalar rewards. Consequently, even with improved initialization, the learned reward model often yields a flat, gradient-free landscape in sparse regions, failing to provide directional guidance for planning. Other works, such as DEMO$^{3}$ \cite{demo3}, rely on manually designed stage-wise rewards, which limits applicability. DreamSmooth \cite{dreamsmooth} tackles long-horizon sparse-reward tasks by predicting temporally smoothed rewards. LS-Imagine \cite{LS-Imagine} achieves strong performance in episodic sparse-reward tasks by constructing a long-short term integrated world model augmented with intrinsic rewards. In contrast, our proposed SLOPE framework fundamentally shifts from predicting scalars to constructing an optimistic potential landscape. By employing optimistic distributional regression, SLOPE ensures sufficient planning gradients even in sparse settings without requiring manual reward engineering. Detailed comparisons are presented in Table \ref{table1:baselines}.

\begin{table}[t!]
\small
\caption{Comparison of MBRL algorithmic features under sparse-reward settings. We contrast the reward modeling paradigms and supervision types of SLOPE against major baselines, highlighting how each method transforms or utilizes sparse environmental signals.}
\label{table1:baselines}
\centering
\setlength{\tabcolsep}{4.0pt} 
\begin{tabular}{lcccc}
\toprule  
\textbf{Method} & \textbf{Reward Shaping} & \textbf{Demonstrations} & \textbf{Learning Target} & \textbf{Reward Supervision} \\ 
\midrule 
\rowcolor{graybg} 
Dreamer-v3 \cite{dreamerv3} & \ding{55} & \ding{55} & Ground-truth reward & Sparse \\
TD-MPC2 \cite{tdmpc2}       & \ding{55} & \ding{55} & Ground-truth reward & Sparse \\
\rowcolor{graybg} 
MoDem \cite{modem}          & \ding{55} & \ding{51} & Ground-truth reward & Sparse \\
DreamSmooth \cite{dreamsmooth} & \ding{51} & \ding{55} & Smoothed reward & Dense + Sparse \\
\rowcolor{graybg} 
DEMO$^{3}$ \cite{demo3}     & \ding{51} & \ding{51} & Stage-wise reward & Semi-sparse \\ 
\midrule 
\rowcolor{bluebg} 
\textbf{SLOPE (Ours)}       & \ding{51} & \ding{51} & \textbf{Potential-based reward} & \textbf{Sparse} \\ 
\bottomrule 
\end{tabular}
\vskip -0.1in
\end{table}


\subsection{Reward Shaping and Exploration}

PBRS \cite{ng_shaping} reshapes the reward and value landscape to provide additional feedback under sparse rewards while preserving optimal policies for fixed, bounded state potentials. This invariance alone does not guarantee improved sample efficiency. Wiewiora \cite{wiewiora2003potential} relates fixed-potential PBRS to Q-value initialization without specifying online potential learning. Later work studies dynamic potentials \cite{dynamic_pbrs}, foundation-model value priors \cite{RLFP}, and potentials derived from learned values \cite{BSRS}. Early model-based work also integrates PBRS \cite{asmuth2008potential}.

Related exploration strategies include entropy regularization \cite{entropy_based_1, entropy_based_2}, visitation bonuses \cite{count_based_1, count_based_2, count_based_3}, and similarity-based signals \cite{similarity_based_1, similarity_based_2, jiang2025episodic}. RND \cite{rnd} uses random-feature prediction errors, whereas Plan2Explore \cite{plan2explore} uses dynamics-ensemble disagreement. Such novelty and disagreement signals need not directly reflect task progress. POLO \cite{polo} constructs optimistic terminal values from a randomized-prior ensemble for MPC exploration, without using them as PBRS supervision for a learned reward model. DOIE \cite{lobel2022optinit} uses state--action knownness to maintain an optimistic value envelope and incorporates domain knowledge to guide continuous-control exploration.

SLOPE constructs a self-evolving potential from a slow-moving target critic for sparse-reward visual MBRL, treating the potential as fixed within each planning and optimization phase. Unlike POLO's epistemic optimism, its upper-quantile weighting emphasizes rare high-return targets amid frequent failures. Potential differences provide dense reward-model supervision for guiding imagined trajectories, without hand-designed potentials or predefined subgoals.

\section{Preliminaries}
\label{preliminaries}
This section explains the basic notation in modeling RL problems and the TD-MPC2 framework, a MBRL algorithm on which our work is built. We then analyze the theoretical limitations of standard planning methods in sparse-reward tasks, which motivate the development of SLOPE.

\subsection{Problem Formulation}

We consider learning in a standard Markov Decision Process (MDP) formulated by a tuple of $\mathcal{M}= \left<\mathcal{S},\mathcal{A},r,\mathcal{P},\gamma \right>$ where $\mathcal{S}$ is the state space, $\mathcal{A}$ is the action space, and $r$ is a sparse reward function. This means that the rewards are typically binary, taking a value of $\left\{ 0,1 \right\}$ where a reward of 1 is received only upon successful completion of the task, and 0 otherwise. $\mathcal{P}\left( \mathbf{s}_{t+1}\left| \mathbf{s}_t,\mathbf{a}_t \right. \right)$ is the state transition function, and $\gamma \in \left[ 0,1 \right)$ is the discount factor to trade off immediate and future rewards. The goal of RL is to find an optimal policy $\pi ^{\ast}$ to maximize the expected cumulative return, $\pi ^{\ast}=\text{arg}\max _{\pi}\mathbb{E}_{\mathbf{a}_t\sim \pi \left( \cdot |\mathbf{s}_t \right) ,\mathbf{s}_t\sim \mathcal{P}}\left[ \sum\limits_{t=1}^T{\gamma ^tr\left( \mathbf{s}_t,\mathbf{a}_t \right)} \right]$, starting from an initial state $\mathbf{s}_0\in \mathcal{S}$ and following a policy $\pi _{\theta}\left( \cdot |\mathbf{s}_t \right)$ that is parameterized by a set of learnable parameters $\theta$.

\subsection{TD-MPC2}

We build our method on \textbf{TD-MPC2} \cite{tdmpc2}, a scalable MBRL algorithm that extends the TD-MPC \cite{tdmpc} framework. In this work, we focus on the single-task formulation. TD-MPC2 learns a latent-space world model comprising five core components: an encoder $h_{\theta}(\mathbf{o})$ mapping observations to latent states $\mathbf{z}$; a dynamics model $d_{\theta}(\mathbf{z}, \mathbf{a})$ predicting future states $\mathbf{z}'$; a reward model $R_{\theta}(\mathbf{z}, \mathbf{a})$; a value function $Q_{\theta}(\mathbf{z}, \mathbf{a})$; and a prior policy $\pi_{\theta}(\mathbf{z})$.

The value function $Q_{\theta}$ is trained via the Bellman operator $\mathcal{T}$:
\begin{equation}
\label{eq:bellman}
    \mathcal{T}Q(\mathbf{z},\mathbf{a}) = R(\mathbf{z},\mathbf{a}) + \gamma \mathbb{E}_{\mathbf{z}', \mathbf{a}'}[Q(\mathbf{z}',\mathbf{a}')]
\end{equation}
During inference, TD-MPC2 performs planning via Model Predictive Path Integral (MPPI) \cite{mppi}. It iteratively samples action sequences from a Gaussian distribution $\mathcal{N}(\mu, \sigma^2)$ and updates the distribution parameters using the top-$k$ trajectories to maximize the expected cumulative return $J(\tau)$:
\begin{align}
\label{eq:mppi_obj}
\mu^*, \sigma^* = \arg\max_{(\mu, \sigma)} \mathbb{E}_{\mathbf{a}_{t:t+H} \sim \mathcal{N}(\mu, \sigma^2)} \left[ J(\tau) \right],
\end{align}
where $J(\tau) = \sum_{h=t}^{H-1} \gamma^h R(\mathbf{z}_h, \mathbf{a}_h) + \gamma^H Q(\mathbf{z}_{t+H}, \mathbf{a}_{t+H})$.

\paragraph{The Challenge of Sparse Rewards.} Standard MBRLs are based on the assumption that $R_{\theta}$ and $Q_{\theta}$ provide dense signals to guide MPPI optimization. However, in sparse reward tasks, this assumption is broken down. 

\begin{figure*}
    \centering
    \includegraphics[width=1.0\linewidth]{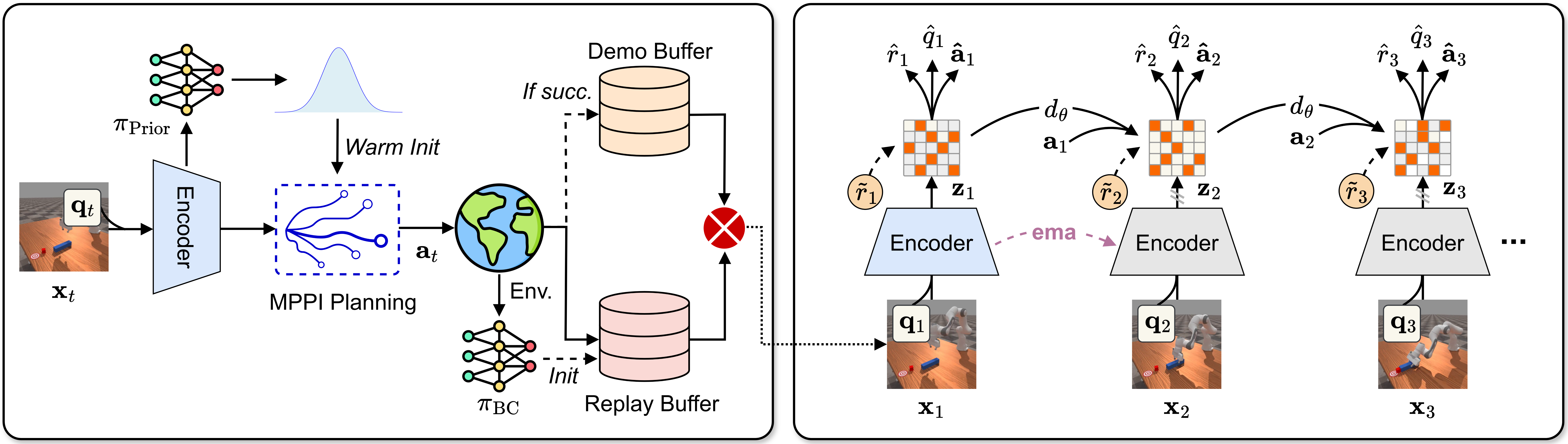}
    \caption{Framework of SLOPE. Building upon MoDem’s multi-phase accelerated learning framework, we also introduce two training enhancements: (i) initializing MPPI sampling distribution from the prior policy $\pi_{\text{Prior}}$, and (ii) augmenting the demonstration buffer with successful trajectories. The shaped reward $\widetilde{r}$ is used for training both the reward model $R_{\theta}$ and the $Q$ value function $Q_{\theta}$. The environment input consists of multimodal observations $\mathbf{o}=(\mathbf{x,q})$, where $\mathbf{x}$ denotes raw RGB images captured by the robot’s camera, and $\mathbf{q}$ represents the robot’s proprioceptive sensory information, e.g., the gripper state. Subsequent observations are encoded by the target encoder (illustrated in grey) which is updated via an exponential moving average (EMA) of the online encoder.}
    \label{fig:framework}
    \vskip -0.1in
\end{figure*}

Consider a scenario in which the goal is unreachable within the horizon $H$. With ground truth rewards $r_t = 0$ for all sampled trajectories, the learned model collapses to zero and $Q_{\theta}$ fails to differentiate between the superior and inferior states. TD-MPC2 employs a variant of MPPI that updates the sampling distribution parameters using only the top-$k$ trajectories with the highest estimated returns. The mean $\mu$ is updated at iteration $j$ as:
\begin{equation}
    \mu^{j+1} = \frac{\sum_{i=1}^{k} \Omega_i \mathbf{a}_i^*}{\sum_{i=1}^{k} \Omega_i}, \quad \text{where } \Omega_i = e^{\kappa \cdot J(\tau_i^*)}.
\end{equation}
Here, $\{\mathbf{a}_i^*\}_{i=1}^k$ represents the action sequences of the top-$k$ elites sorted by their cumulative return $J(\tau)$, and $\kappa$ is a temperature parameter. In a sparse setting where $J(\tau) \approx 0$ across all $N$ sampled trajectories, the top-$k$ selection degenerates into a random subset of the original noise. Furthermore, the importance weights become uniform. Consequently, the update $\mu^{j+1}$ essentially averages a random selection of noise, preventing the planner from shifting the search distribution towards the goal. This theoretical limitation motivates our shift from fitting sparse scalar rewards to learning a potential landscape.

\section{Methodology}

In this section, we introduce \textbf{SLOPE}, a specialized framework that extends PBRS through optimism-driven landscape construction. Rather than relying on standard scalar reward regression, SLOPE formulates a dense potential landscape to provide continuous guidance in model-based planning. The method comprises three core components: \textbf{Model-Based Planning via Potential Landscape}, which leverages PBRS guarantees to embed directional guidance directly into the planner's objective; \textbf{Optimism-Driven Landscape Shaping}, which approximates a quantile-driven optimistic distributional Bellman operator via asymmetric weighting to capture upper-quantile returns and drive exploration; and \textbf{Accelerated Training Strategies} to ensure stable and efficient convergence.

\begin{wrapfigure}{r}{0.5\textwidth}
  \vspace{-15pt} 
  \centering
  
  \begin{subfigure}[b]{0.215\textwidth} 
    \centering
    \includegraphics[width=\textwidth, trim=0cm 0.9cm 0cm 0cm, clip]{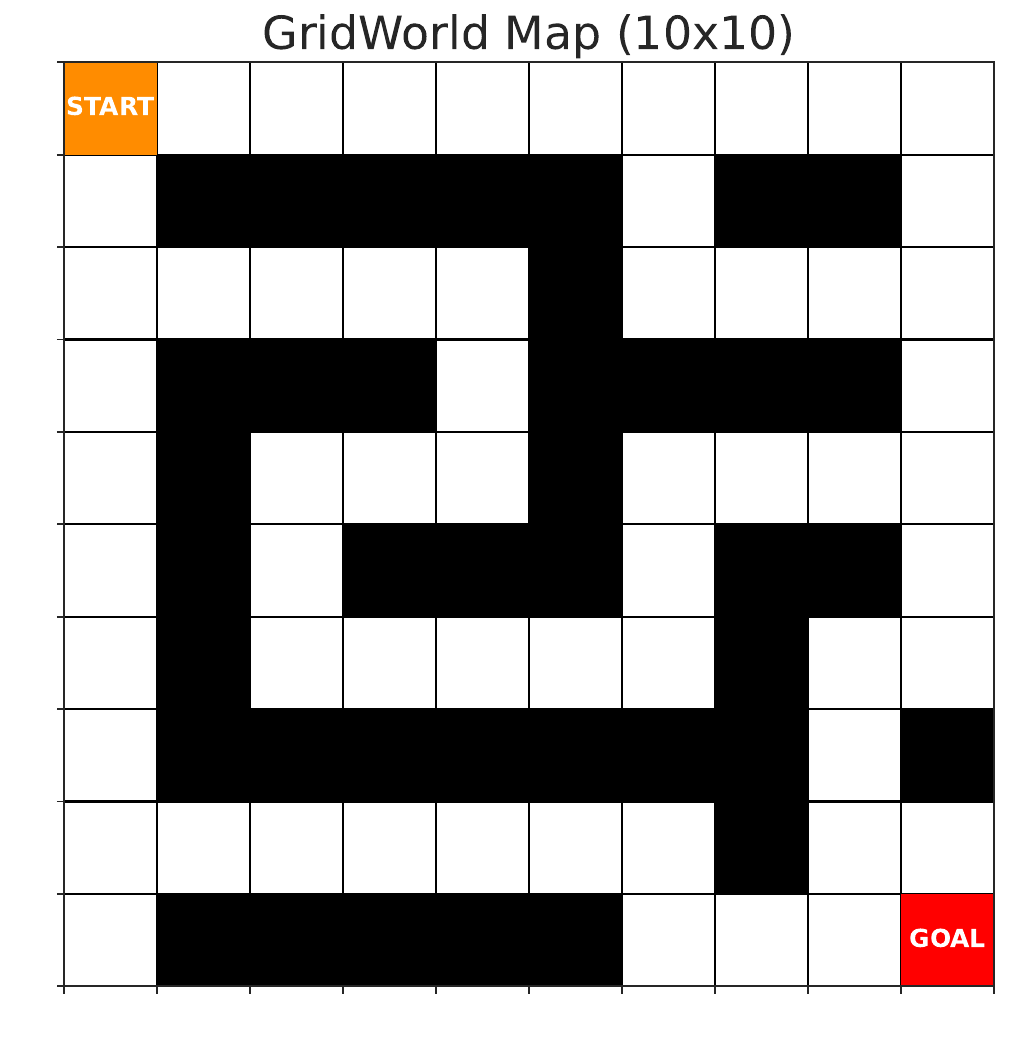}
    \label{fig:gridmap}
  \end{subfigure}
  \hfill 
  \begin{subfigure}[b]{0.245\textwidth} 
    \centering
    \includegraphics[width=\textwidth, trim=0cm 1.65cm 0cm 1.25cm, clip]{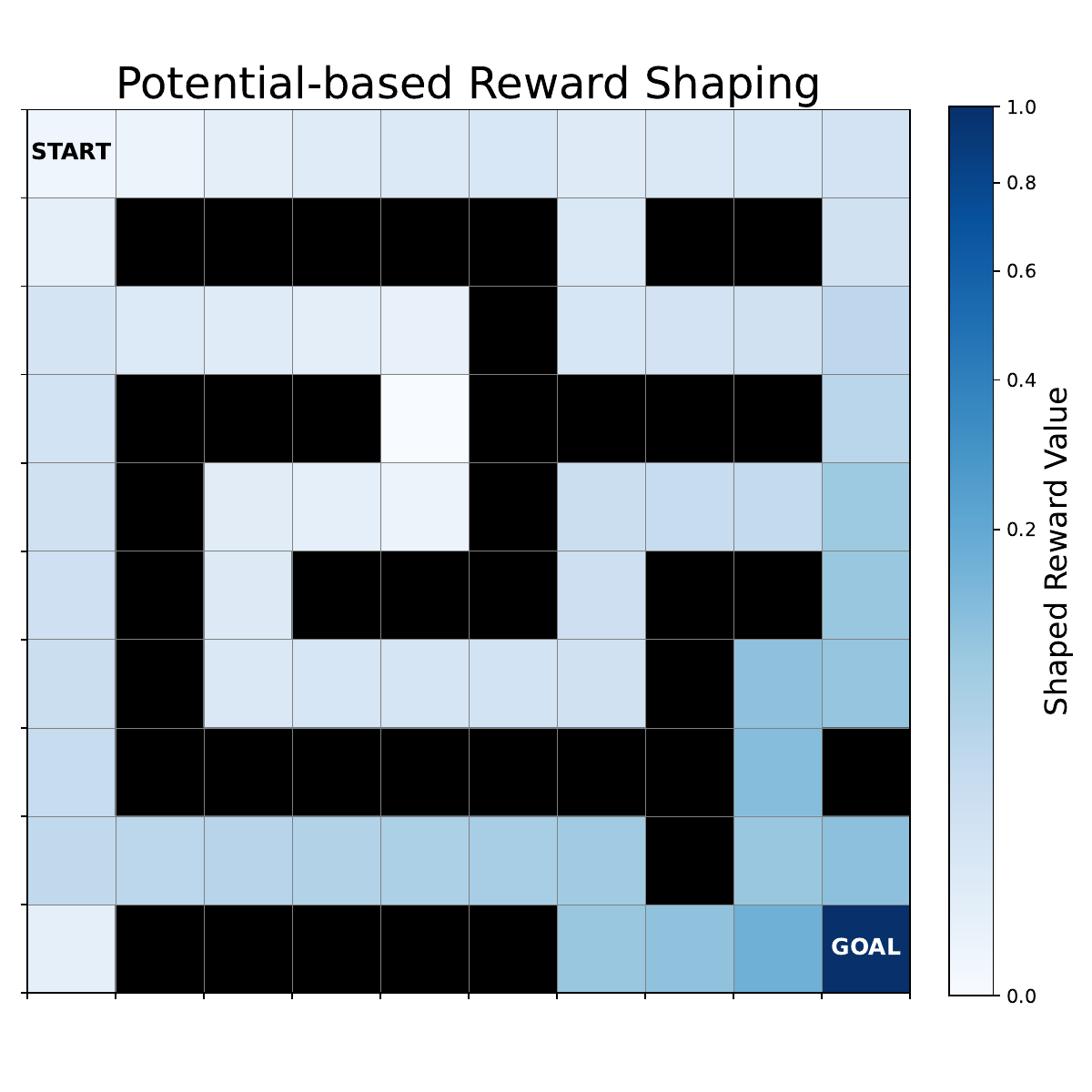}
    \label{fig:shaped_reward}
  \end{subfigure}
  
  \vskip -0.1in
  \caption{Toy example on a $10 \times 10$ GridWorld. \textbf{Left}: The original environment with sparse rewards ($+1$ at goal, $0$ elsewhere). \textbf{Right}: The dense reward signal generated by PBRS using converged optimal value function $V^*$. \protect$\blacksquare$: wall.}
  \label{fig:toy_example}
  \vspace{-10pt} 
\end{wrapfigure}

\subsection{Model-Based Planning via Potential Landscape}
\label{sec:pbrs}
Standard MBRL methods are based on planning over imagined trajectories by maximizing the sum of predicted scalar rewards. However, in sparse reward settings, this reliance on immediate scalar feedback creates a fundamental gradient starvation problem. As illustrated in the $10 \times 10$ GridWorld example (Fig.\ref{fig:toy_example}, Left), a standard reward model effectively predicts zero everywhere outside the goal state. This topological flatness means $R_\theta$ provides negligible gradients, depriving the latent encoder and dynamics model of crucial representation learning signals. Even with a terminal value term (Eq.~\ref{eq:mppi_obj}), the planner's overall objective landscape remains uninformative, depriving sampling-based planners like MPPI of the directional information necessary to distinguish between stagnant states and progressive paths.

To overcome this structural bottleneck, our approach transitions reward modeling from predicting sparse rewards to constructing a dense potential landscape. Instead of learning a flat $R_\theta$, we formulate a reshaped reward model $\widetilde{R}_\theta(\mathbf{z}, \mathbf{a})$ based on PBRS \citep{ng_shaping}:
\begin{theorem}[Potential-Based Reward Shaping \citep{ng_shaping}]
\label{thm:ng_shaping}
    Given the MDP $\mathcal{M}$ with optimal policy $\pi^*$, the reshaped MDP $\widetilde{\mathcal{M}}$ with the reward function defined as:
    \begin{equation}
    \label{eq:shaped_reward}
    \widetilde{r}(\mathbf{s},\mathbf{a}) = {r}(\mathbf{s},\mathbf{a}) + \gamma \mathbb{E}_{\mathbf{s}'\sim \mathcal{P}}\left[ \Phi \left( \mathbf{s}' \right) \right] - \Phi(\mathbf{s})
    \end{equation}
    preserves the optimal policy, i.e., $\widetilde{\pi}^*= \pi^*$, for any bounded potential function $\Phi \colon \mathcal{S} \to \mathbb{R}$. Furthermore, the optimal value functions satisfy:
    \begin{equation}
    \widetilde{V}^*(\mathbf{s}) = {V}^*(\mathbf{s}) - \Phi(\mathbf{s}), \quad \widetilde{Q}^*(\mathbf{s},\mathbf{a}) = {Q}^*(\mathbf{s},\mathbf{a}) - \Phi(\mathbf{s}).
    \end{equation}
\end{theorem}
The topological quality of this reshaped landscape heavily depends on the choice of $\Phi$. Ideally, as shown in Fig.~\ref{fig:toy_example} (Right), setting $\Phi$ to the true optimal value function $V^*$ provides perfect directional guidance. By replacing the sparse target $r$ with such a dense potential difference, the reward model is forced to encode local value gradients at \textit{every single time step}. By predicting this local trend of the potential landscape instead of merely sparse scalar values, the reward model learns to capture the trend of the potential landscape, providing the planner with dense, step-by-step directional guidance.

Since the optimal value is unknown a priori, we employ a self-evolving landscape strategy \cite{BSRS}. To stabilize bootstrapping and reduce fluctuations in the shaping signal, we construct the potential from the agent's slow-moving target value network:
\begin{equation}
    \Phi_k(\mathbf{s})
    = \eta \max_{\mathbf{a}}
      Q_{\bar{\theta}_k}(\mathbf{s}, \mathbf{a}),
\end{equation}
where $\eta$ is a scaling factor and $\bar{\theta}_k$ denotes the target parameters at iteration $k$, updated across iterations via exponential moving average (EMA). The resulting potential differences provide transition-level supervision for the learned reward model, enabling dense guidance throughout imagined rollouts.

Although the potential evolves across training iterations, the policy-invariance guarantee applies separately to each snapshot in which the target parameters are held fixed during the corresponding planning and optimization phase. Under this fixed-potential assumption, the following corollary
follows directly from Theorem~\ref{thm:ng_shaping}.
\begin{corollary}[Per-Iteration Policy Invariance]
\label{prop:policy_invariance}
Consider a discounted MDP $\mathcal{M}$ with $\gamma \in [0,1)$. At iteration $k$, let
\[
    \Phi_k(\mathbf{s})
    = \eta \max_{\mathbf{a}}
      Q_{\bar{\theta}_k}(\mathbf{s},\mathbf{a})
\]
be a bounded state potential, with the target parameters $\bar{\theta}_k$ held fixed during the corresponding planning and optimization phase. Let $\widetilde{\mathcal{M}}_k$ have the same states, actions, transition dynamics, and discount factor as $\mathcal{M}$, but use the shaped reward
\[
    \widetilde{r}_k(\mathbf{s},\mathbf{a})
    = r(\mathbf{s},\mathbf{a})
    + \gamma \mathbb{E}_{\mathbf{s}'\sim
      \mathcal{P}(\cdot\mid\mathbf{s},\mathbf{a})}
      \bigl[\Phi_k(\mathbf{s}')\bigr]
    - \Phi_k(\mathbf{s}).
\]
Then,
\[
    \Pi^*(\widetilde{\mathcal{M}}_k)
    = \Pi^*(\mathcal{M}),
    \qquad \forall k,
\]
where $\Pi^*(\cdot)$ denotes the set of optimal policies. Thus, although the potential evolves across iterations, each fixed-potential snapshot preserves the optimal-policy set of the original MDP.
\end{corollary}

\begin{proof}
For each fixed iteration $k$, $\Phi_k$ is a bounded state potential, so the result follows directly from Theorem~\ref{thm:ng_shaping}.
\end{proof}

\subsection{Optimism-Driven Landscape Shaping}
\label{sec:odbo}
Although PBRS offers a mechanism for dense feedback, its efficacy depends on the quality of the value estimates. Standard regression techniques focusing on the expectation frequently suppress rare success signals in sparse reward tasks, yielding a flat potential landscape.

To address this, we employ an \textbf{Optimism-Driven Value Learning} mechanism to actively shape the potential landscape. By introducing an optimistic bias, this approach amplifies the impact of rare successful trajectories against the background of failure. Consequently, the constructed potential landscape retains informative gradients pointing towards high-reward regions.

We implement this mechanism via \textbf{Quantile-weighted Cross-Entropy (QCE)} loss. Unlike standard distributional approaches (e.g., TD-MPC2) that use symmetric objectives, the QCE loss imposes an asymmetric penalty on prediction errors. We parameterize the action-value function $Q_\theta(\mathbf{s}, \mathbf{a})$ as a categorical distribution over discrete support bins, where the scalar $Q$-value is obtained via the expectation $\mathbb{E}[Q_\theta(\mathbf{s}, \mathbf{a})]$.

Formally, the regression target $y$ is first computed using the standard scalar expectation of the target network:
\begin{equation}
    y = \widetilde{r}(\mathbf{s}, \mathbf{a}) + \gamma \, \mathbb{E}_{\mathbf{s}' \sim \mathcal{P}} \left[ \max_{\mathbf{a}'} \mathbb{E}[Q_{\bar{\theta}}(\mathbf{s}', \mathbf{a}')] \right].
\end{equation}
Following TD-MPC2 \cite{tdmpc2}, we project the scalar target $y$ into a distributional representation $\Psi(y)$ using soft two-hot encoding. This technique distributes probability mass across the fixed discrete support bins via linear interpolation. To enforce optimism, we define the QCE objective as:
\begin{equation}
\label{eq:qce_loss}
    \mathcal{L}_{\text{QCE}}(\theta) = \mathbb{E}_{\mathcal{D}} \left[ w(\delta) \cdot \text{CE}\left( \Psi(y), Q_\theta(\mathbf{s}, \mathbf{a}) \right) \right].
\end{equation}
Here, $\text{CE}(\cdot, \cdot)$ denotes the Cross-Entropy loss, and $w(\delta)$ is a dynamic weighting term based on the scalar estimation error $\delta = \mathbb{E}[Q_\theta(\mathbf{s}, \mathbf{a})] - y$:
\begin{equation}
    w(\delta) = \tau \cdot \mathbbm{1}_{\{\delta < 0\}} + (1 - \tau) \cdot \mathbbm{1}_{\{\delta \geq 0\}},
\end{equation}
where $\tau \in (0.5, 1]$ represents the quantile coefficient. By setting $\tau > 0.5$, the objective penalizes underestimation ($\delta < 0$) more heavily than overestimation. This asymmetry drives the learned value distribution towards the upper quantiles of the return distribution, effectively encouraging exploration in uncertain but potentially high-reward regions.

To theoretically guarantee the stability of this optimistic update, we formalize it as dynamic reward shaping and prove its convergence properties below.

\begin{proposition}[Contraction Mapping under Target-Anchored Shaping]
\label{prop:contraction_target}
Consider the reshaped MDP $\widetilde{\mathcal{M}} = \left< \mathcal{S}, \mathcal{A}, \widetilde{r}, \mathcal{P}, \gamma \right>$ with the reward dynamically modified by the target-anchored potential function. Let $\Phi(\mathbf{s}) = \eta \max_{\mathbf{a}} Q_{\bar{\theta}}(\mathbf{s}, \mathbf{a})$, where $Q_{\bar{\theta}}$ is a target network held fixed during the current planning and optimization step. The reshaped reward is defined as:
\[
\widetilde{r}_{\Phi}(\mathbf{s}, \mathbf{a}) = r(\mathbf{s}, \mathbf{a}) + \gamma \mathbb{E}_{\mathbf{s}' \sim \mathcal{P}}[\Phi(\mathbf{s}')] - \Phi(\mathbf{s})
\]
For any scaling factor $\eta > 0$, the Bellman Operator $\mathcal{T}_{\Phi}$ applied to this target-anchored reshaped MDP is a strict $\gamma$-contraction in the $L_\infty$ norm.
\end{proposition}

\begin{proof}
See Appendix~\ref{proof:prop_convergence} for detailed proof.
\end{proof}

\begin{proposition}[Consistent Policy Improvement]
\label{prop:policy_improvement}
Let $\pi_{k+1}$ be derived greedily from the reshaped landscape, i.e., $\pi_{k+1}(\mathbf{s}) = \arg\max_{\mathbf{a}} Q_k^{\pi_k}(\mathbf{s}, \mathbf{a})$. Then, this update monotonically improves the policy not only in the reshaped MDP $\widetilde{\mathcal{M}}_k$ but also with respect to the original objective $\mathcal{M}$, such that $Q_{\mathcal{M}}^{\pi_{k+1}}(\mathbf{s}, \mathbf{a}) \ge Q_{\mathcal{M}}^{\pi_k}(\mathbf{s}, \mathbf{a})$.
\end{proposition}
\begin{proof}
See Appendix~\ref{proof:prop_improvement} for detailed proof.
\end{proof}

\subsection{Accelerating MBRL with Sparse Rewards}
In the early stage of training, a limited amount of high-quality demonstration data is crucial for efficient policy exploration and model learning, while it can also be leveraged to initialize an informative potential landscape. Building upon previous work MoDem \cite{modem}, which accelerates MBRL with limited demonstrations, we adopt its three-phase training framework. Additionally, we also introduce two training enhancements: (1) initializing the MPPI sampling distribution with a prior policy in place of the random distribution to accelerate planning; and (2) augmenting the demonstration buffer with successful trajectories. To mitigate the sample imbalance in sparse-reward tasks, successful trajectories encountered during training is actively incorporated into the demo buffer.

\section{Experiments}
\begin{wrapfigure}[17]{r}{0.50\textwidth}
    \vspace{-0.15in} 
    \centering  
    \includegraphics[width=\linewidth, trim={1cm 0cm 1cm 0cm}, clip]{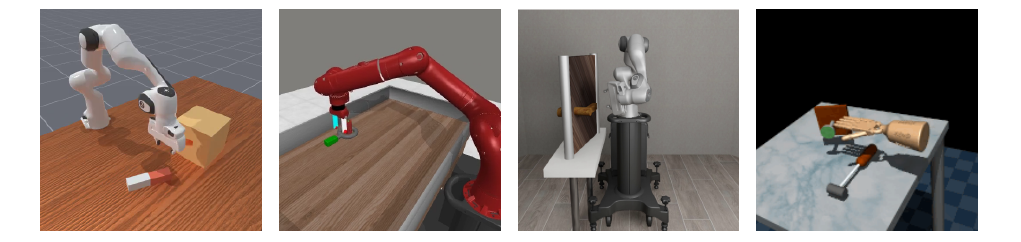}
    \vspace{0.15in} 
    \includegraphics[width=\linewidth]{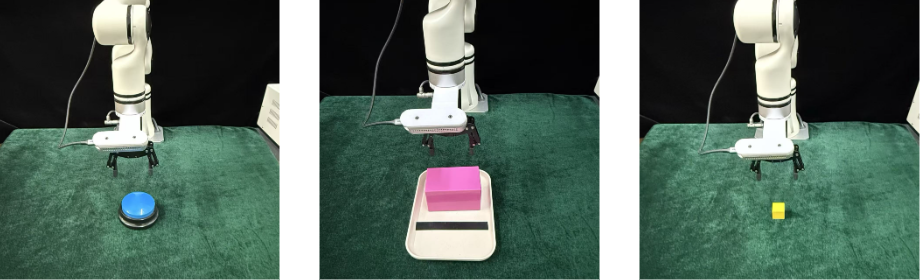}
    \vskip -0.15in
    \caption{Visualizations of representative tasks. \textbf{Top}: Simulation benchmarks including \textit{ManiSkill3} (\texttt{ms}), \textit{Meta-World} (\texttt{mw}), \textit{Robosuite}, and \textit{Adroit}. \textbf{Bottom}: Real-world tasks including Press Button, Push Cube, and Grasp Cube.}
    \label{fig:tasks_vis}
\end{wrapfigure}
We evaluate our method on five continuous control benchmarks and additionally validate it in real-world sparse-reward robotic tasks (see Fig.\ref{fig:tasks_vis}). Our suite includes sparse-reward tasks from \textit{ManiSkill3} \cite{maniskill3} (5 tasks), \textit{Meta-World} \cite{metaworld} (10 tasks, categorized as \texttt{medium}, \texttt{hard}, and \texttt{very hard} by \cite{mwm}), \textit{Robosuite} \cite{robosuite} (3 tasks), and \textit{Adroit} \cite{adroit} (2 tasks), alongside dense-reward tasks from \textit{DeepMind Control Suite} \cite{dmc}. Through these experiments, we aim to investigate the following key questions: (1) \textit{Can our proposed method improve MBRL algorithms under sparse reward settings?} (2) \textit{Can SLOPE learn discriminative reward signals that accurately reflect task progress to better guide policy learning?} (3) \textit{How important are the individual components of SLOPE?} (4) \textit{How does SLOPE compare with intrinsic-motivation and disagreement-based exploration methods?}

\subsection{Experimental Settings}

\begin{figure}[t!]
    \centering
    \includegraphics[width=1.0\linewidth]{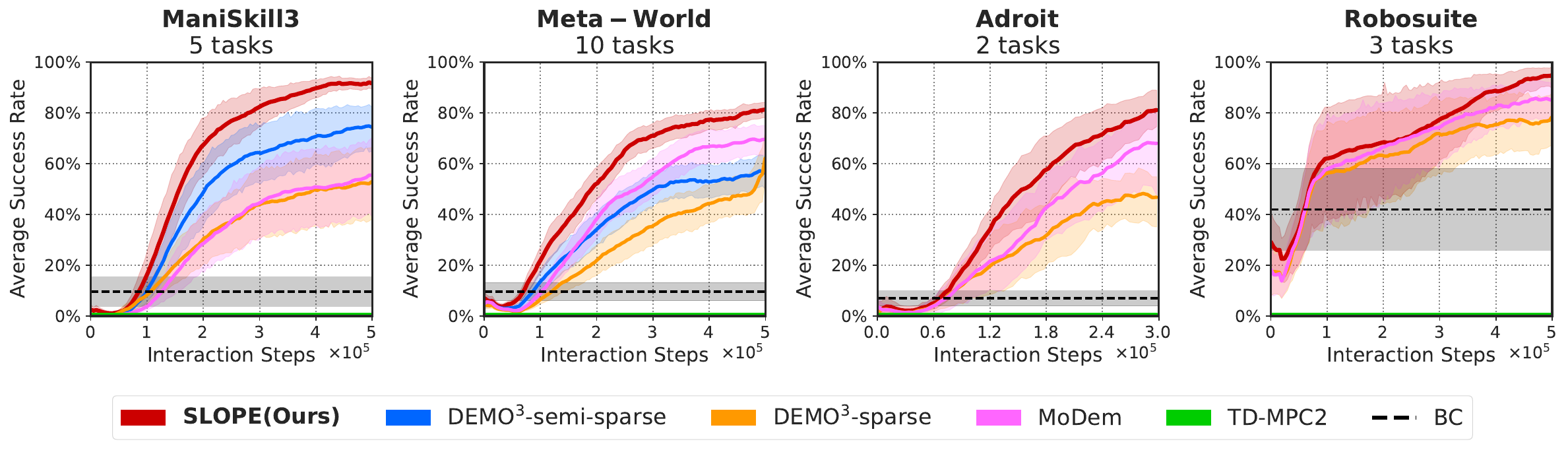}
    \caption{Average success rates across 20 tasks from 4 benchmarks. Curves and shaded areas represent the mean and 95\% confidence intervals (CIs) over 5 independent runs. We include both sparse and semi-sparse tasks for \textit{ManiSkill3} and \textit{Meta-World}, while \textit{RoboSuite} and \textit{Adroit} involve sparse rewards only, following their native task designs. See Appendix \ref{sec:detail_results} for individual task details.}
    \label{fig:average}
\end{figure}

\paragraph{Baselines.} To evaluate the effectiveness of SLOPE, we compare it against several closely related MBRL methods: (1) \textbf{TD-MPC2} \cite{tdmpc2}, a strong and widely adopted baseline known for its scalability and multi-task performance. We focus on its effectiveness under sparse reward settings. (2) \textbf{MoDem} \cite{modem}, a demonstration-augmented MBRL method with a three-phase training pipeline. Originally built on TD-MPC, we reimplement it with TD-MPC2 for fair comparison. MoDem achieves SOTA results on sparse-reward benchmarks like \textit{Meta-World} and \textit{Adroit}, highlighting its sample efficiency from expert demonstrations. (3) \textbf{DEMO${^3}$} \cite{demo3}, An MBRL method that reformulates sparse-reward tasks into manually defined multi-stage sub-tasks with denser intermediate rewards (semi-sparse). DEMO${^3}$ has reported SOTA results on complex manipulation domains such as \textit{ManiSkill3} and \textit{Robosuite}. However, some tasks are inherently single stage or cannot be decomposed into multiple stages. We evaluate DEMO${^3}$ under both semi-sparse (DEMO${^3}$-semi-sparse) and fully sparse (DEMO${^3}$-sparse) settings to assess its adaptability to different reward structures. (4) \textbf{Behavior Cloning} (BC), a supervised imitation learning baseline that directly learns policies from expert demonstrations. 

We also implement three exploration baselines on the same MoDem framework with a TD-MPC2 backbone: (1) \textbf{Latent RND} \cite{rnd}, whose intrinsic reward is the prediction error of a trainable network predicting the output of a fixed randomly initialized network from the next-state latent
representation; (2) \textbf{Q-ensemble disagreement (QD)}, which uses the standard deviation of scalar predictions from TD-MPC2's five target Q-heads as an exploration bonus, capturing value disagreement rather than dynamics disagreement; and (3) \textbf{Plan2Explore-style disagreement (P2E)} \cite{plan2explore}, which uses disagreement among five auxiliary one-step latent-dynamics predictors trained with bootstrap masks. For each baseline, the normalized bonus is added to the task reward for reward-model and value learning. All variants use the same demonstrations, backbone, interaction budget, and evaluation protocol.
We compare them on Pick Place and Peg Insertion, two challenging sparse-reward ManiSkill3 tasks on which MoDem alone performs poorly.

\begin{table}[t!]
    \centering
    \small
    \caption{Comparison of methods on real-world tasks at 15\textit{k} environment steps ($\approx$ 1 hour of real-world training). Results are averaged over 20 evaluation trials. Success Rate and Intervention Rate are reported in percentage (\%). The best results are highlighted in \hl{\textbf{colorbox}}.}
    \label{tab:methods_comparison}
    \setlength{\tabcolsep}{3.0pt}
    \begin{tabular}{lcccccc}
        \toprule
        \multirow{2}{*}{\textbf{Method}} & \multicolumn{2}{c}{\textbf{Press Button}} & \multicolumn{2}{c}{\textbf{Push Cube}} & \multicolumn{2}{c}{\textbf{Grasp Cube}} \\
        \cmidrule(lr){2-3} \cmidrule(lr){4-5} \cmidrule(lr){6-7}
         & Success ($\uparrow$) & Intervention ($\downarrow$) & Success ($\uparrow$) & Intervention ($\downarrow$) & Success  ($\uparrow$) & Intervention ($\downarrow$) \\
        \midrule
        SAC                   & 45.0 & 16.8 & 0.0 & 46.5 & 0.0 & 45.3 \\
        MoDem                 & 65.0 & 10.0 & 15.0 & 18.4  & 35.0 & 24.0 \\
        DEMO${^3}$            & \hl{\textbf{100.0}} & 11.1 & 15.0 & 19.1  & 20.0 & 23.2 \\
        \midrule
        \textbf{SLOPE (Ours)} & \hl{\textbf{100.0}} & \hl{\textbf{7.3}}  & \hl{\textbf{65.0}} & \hl{\textbf{18.2}}  & \hl{\textbf{65.0}} & \hl{\textbf{21.3}}\\
        \bottomrule
    \end{tabular}
    \vskip -0.1in
\end{table}

\paragraph{Reward Settings.} In this work, we consider: (1) \textbf{sparse} rewards, where $r \in \{0,1\}$ is given only upon task completion; (2) \textbf{semi-sparse} rewards, formulated as a stage indicator $r: S \rightarrow \{1,\dots,N\}$; and (3) \textbf{dense} rewards, included as a supplementary benchmark.

\paragraph{Real Robot Settings.} Real-robot experiments are conducted using a Realman RM75-B robotic arm equipped with two RGB cameras to capture image observations. Following the Human-in-the-Loop Sample Efficient Robotic Reinforcement Learning (HIL-SERL) pipeline \cite{hil-serl}, we train SLOPE entirely from scratch in the real world without any simulation pre-training. We evaluate our method on three sparse-reward tasks: \texttt{Press Button}, \texttt{Push Cube}, and \texttt{Grasp Cube}. The tasks are illustrated in Fig.\ref{fig:tasks_vis}. Additional real-robot implementation details are available in the Appendix \ref{appendix:real_robot}.

\subsection{Simulation Experiments Results}
As shown in Fig.\ref{fig:average}, the pure MBRL baseline TD-MPC2 fails across all tasks under sparse reward settings. While the BC policy achieves decent initial performance, it is limited by the amount of available expert data and struggles to improve further. In contrast, SLOPE demonstrates superior performance across all four benchmarks using only sparse rewards. On \textit{ManiSkill3}, it significantly outperforms the previous SOTA, DEMO${^3}$-semi-sparse. Unlike DEMO$^3$-sparse, which suffers a notable drop without handcrafted stage priors, SLOPE consistently surpasses competing methods (including MoDem and DEMO$^3$ variants) on \textit{Meta-World}, \textit{Adroit}, and \textit{Robosuite}. These results verify SLOPE’s ability to learn effectively without dense rewards or stage-wise priors.

\subsection{Real-World Experiments Results}
Table \ref{tab:methods_comparison} demonstrates that SLOPE consistently outperforms baselines in both success and efficiency. On the simple \texttt{Press Button} task, SLOPE matches DEMO$^3$ at 100.0\% success but with significantly lower intervention. In complex tasks, SLOPE attains 65.0\% success, surpassing baselines significantly. Notably, on \texttt{Grasp Cube}, SLOPE significantly exceeds both MoDem (35.0\%) and DEMO$^3$ (20.0\%) while maintaining the lowest intervention rate, whereas SAC fails completely. Collectively, these results validate SLOPE's robustness and sample efficiency in solving complex sparse-reward manipulation problems with minimal human intervention.

\subsection{Comparison with Other Exploration Baselines}
\label{sec:exploration_comparison}
\begin{wrapfigure}[19]{r}{0.50\textwidth}
    \centering
    \vskip -0.15in
    \includegraphics[width=\linewidth]{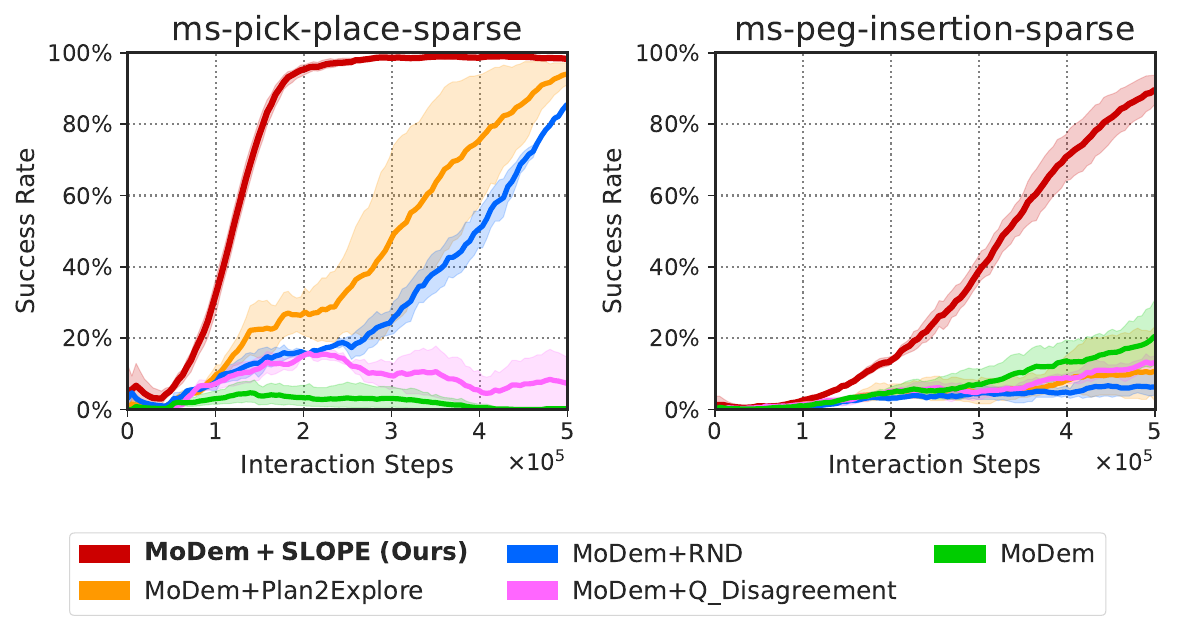}
    \caption{Comparison with intrinsic-motivation and disagreement-based exploration baselines on sparse-reward \textit{ManiSkill3} tasks. All variants share the MoDem backbone and demonstrations. Compared with RND, Q-ensemble disagreement, and Plan2Explore-style disagreement, SLOPE learns faster on Pick Place and achieves substantially higher final success on Peg Insertion.}
    \label{fig:exploration_baselines}
    \vspace{-8pt}
\end{wrapfigure}
Fig.~\ref{fig:exploration_baselines} compares the success rate learning curves over $5\times10^5$ interaction steps. On \texttt{Pick Place}, RND and P2E substantially improve over MoDem, but require considerably more interactions to achieve high success rates than SLOPE. QD provides only a modest and unstable improvement. SLOPE reaches a high success rate early in training and maintains it throughout the remaining budget, demonstrating stronger sample efficiency on this task.

The pattern differs on \texttt{Peg Insertion}, where all three intrinsic baselines underperform MoDem, whereas SLOPE achieves approximately 90\% success rate. \texttt{Peg Insertion} requires precise alignment and carefully controlled contact interactions, leaving only a narrow range of successful behaviors. In this setting, visitation novelty or disagreement may prioritize under-visited or difficult-to-predict transitions that do not reflect meaningful insertion progress, thereby steering the planner away from the demonstration-supported success region. 

These results indicate that generic novelty and uncertainty bonuses can be effective when broad coverage is sufficient, as in \texttt{Pick Place}, but may become counterproductive in precision-sensitive tasks. In contrast, SLOPE derives its guidance from learned task values and yields the strongest performance on both tasks.

\subsection{Further Analysis}
\paragraph{Ablation Study.}To better understand the contribution of each component in SLOPE, we conduct an ablation study comparing the full model against two key variants: (1) \textbf{SLOPE w/o ATS}, which removes the Accelerated Training Strategies to isolate the effect of our core contributions, and (2) \textbf{SLOPE w/o PL\&ODS}, which strips away the Potential Landscape and Optimism-Driven Shaping mechanisms, relying solely on the auxiliary training strategies. The results are shown in Fig.~\ref{fig:additional_results} (Left).

\begin{figure}[t!]
    \centering
    \begin{subfigure}[t]{0.48\linewidth}
        \centering
        \includegraphics[width=\linewidth]{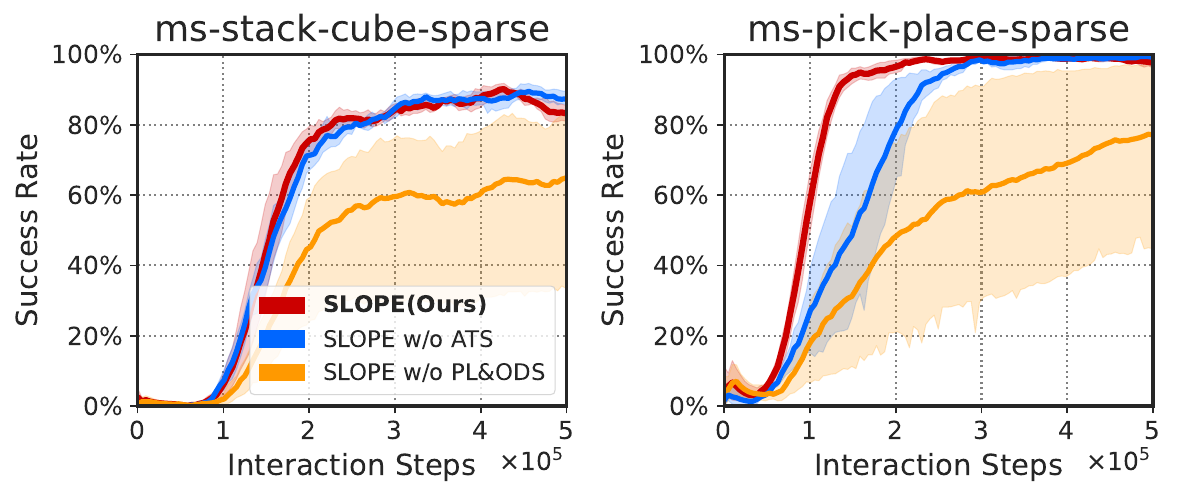}
    \end{subfigure}
    \hfill
    \begin{subfigure}[t]{0.48\linewidth}
        \centering
        \includegraphics[width=\linewidth]{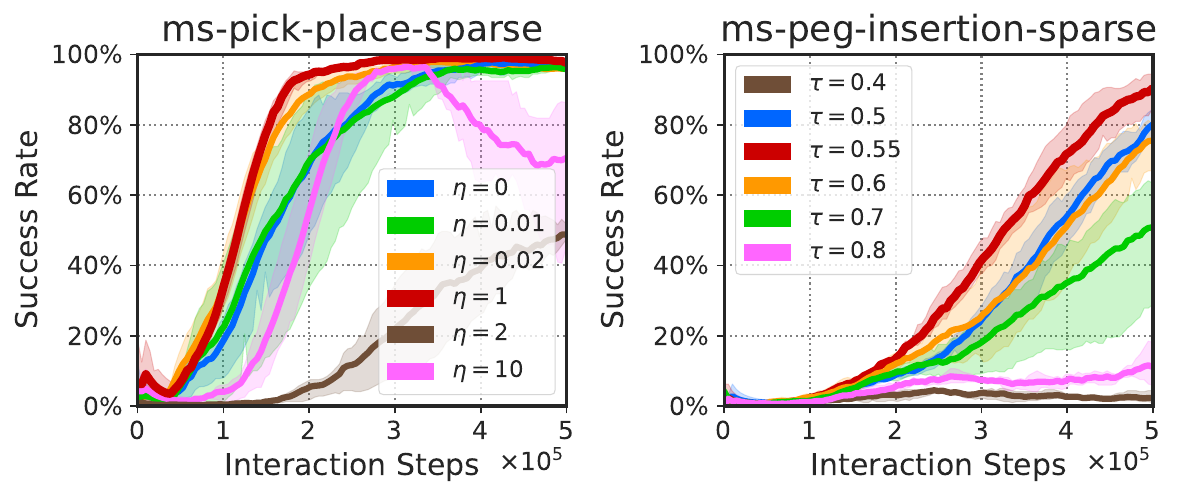}
    \end{subfigure}
   \caption{\textbf{Left}: Ablation study. \textbf{Right}: Hyperparameter sensitivity analysis of SLOPE.}
   \vskip -0.15in
    \label{fig:additional_results}
\end{figure}

Removing potential landscape and optimism-driven shaping mechanisms leads to a severe performance degradation across all tasks, confirming that our core potential-based shaping mechanism is the primary driver for efficient learning under sparse rewards. Conversely, disabling accelerated training strategies only result in a marginal performance decrease, maintaining highly competitive success rates across most tasks. This strongly indicates that our core potential landscape and optimism-driven shaping mechanisms are the primary drivers of the overall performance improvements. Nevertheless, we observe that integrating the auxiliary accelerated training strategies still provide practical benefits by further accelerating convergence and stabilizing the training process.

\paragraph{Hyperparameter Sensitivity.}We evaluate the sensitivity to $\eta$ and $\tau$ in Fig.~\ref{fig:additional_results} (Right). For the shaping weight, we find $\eta = 1.0$ to be the optimal value for providing effective directional guidance. However, performance degrades if $\eta$ continues to increase, as an overly strong shaping signal tends to overshadow the original task reward. Regarding $\tau$, moderate optimism at $\tau = 0.55$ yields the best exploration-exploitation trade-off compared to the unbiased baseline ($\tau = 0.5$). Similarly, excessively high $\tau$ values cause performance drops, as extreme optimism severely biases the value distribution and reduces the reliability of the shaping potential.

\paragraph{Visualization and Analysis}
\begin{wrapfigure}[14]{r}{0.50\textwidth}
    \vspace{-0.2in} 
    \centering  
    \includegraphics[width=\linewidth, trim={0.2cm 0.3cm 0.2cm 0cm}, clip]{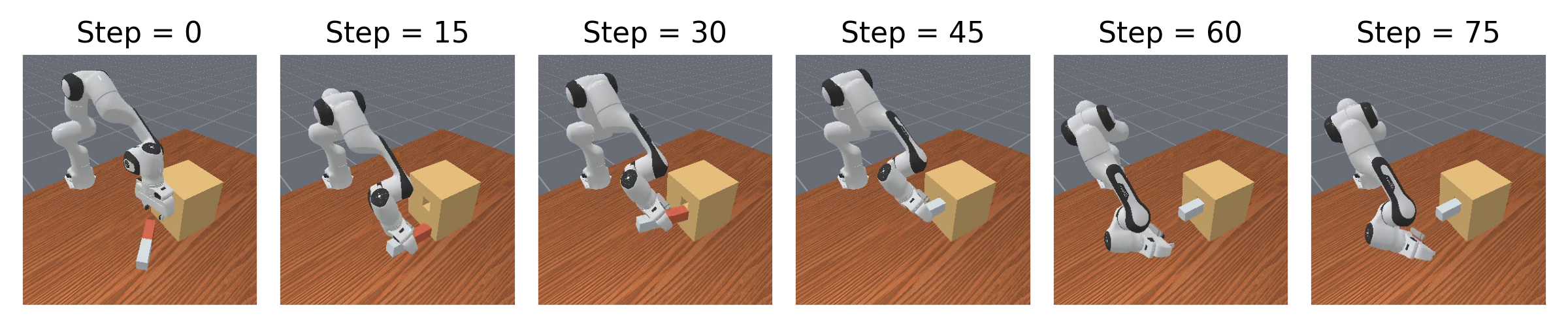}
    \includegraphics[width=\linewidth]{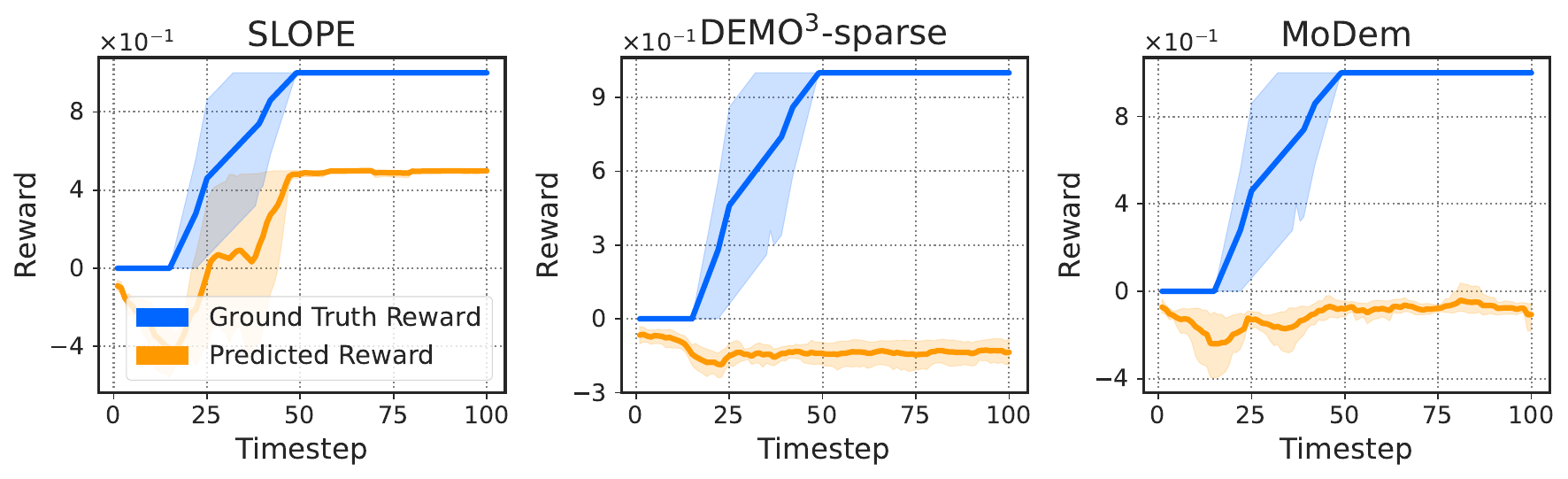}
    \caption{Visualization of reward predictions. \textbf{Top:} Keyframes of task execution. \textbf{Bottom:} Ground truth vs. predicted rewards along the trajectory.}
    \label{fig:reward_vis}
\end{wrapfigure}
To address Question (2), we visualize the reward predictions of different algorithms along a successful trajectory in Fig.\ref{fig:reward_vis}. It can be observed that the reward models of MoDem and DEMO$^3$-sparse produce relatively flat predictions, failing to capture the underlying task structure. In contrast, SLOPE demonstrates a strong correlation with the ground truth reward. Crucially, even in the absence of dense supervision, SLOPE's predicted reward exhibits a clear upward trend that aligns with the temporal progression of task completion. This indicates that our method successfully learns a discriminative reward model capable of providing dense, meaningful guidance within sparse-reward tasks.

See Appendix \ref{sec:additional_results} for comprehensive evaluations, including performance on dense reward tasks, transferability to other MBRL algorithms, and further analysis.

\section{Conclusion}
In this paper, we presented SLOPE, a novel framework that enhances MBRL in sparse-reward tasks by integrating optimistic potential landscape construction into the reward modeling. By utilizing the agent’s value estimates as a self-evolving potential function, SLOPE provides dense directional guidance absent in sparse structures. SLOPE also incorporates a distributional value mechanism to generate optimistic upper bounds, effectively preventing the suppression of rare success signals during exploration. Extensive experiments across 5 challenging benchmarks and real-world tasks demonstrate that SLOPE consistently outperforms SOTA methods.

\paragraph{Limitations and Future Work.} One potential limitation of SLOPE is the inherent tension between optimistic exploration and safety constraints, as aggressive potential scaling may inadvertently encourage high-risk behaviors in unvisited, hazardous regions. To address this, future work will extend our framework into Safety-Aware MBRL by incorporating safety-centric priors, such as Control Barrier Functions (CBF), to construct repulsive potential fields.

\section*{Acknowledgements}
The authors would like to thank the anonymous reviewers for their valuable feedback. This work was partially supported by the National Natural Science Foundation of China (Grant No. 62276024), the Beijing Natural Science Foundation (Grant No. 4262066), the Fundamental Research Funds for the Central Universities, Jilin University (Grant No. 93K172025K01), and the Fundamental Research Funds for the Central Universities (Grant No. 2025CX01010).

\bibliographystyle{unsrt}
\bibliography{ref}

\clearpage

\appendix

\section{Proof}
\label{proof}

\theoremstyle{plain}
\newtheorem*{prop_4.2}{Proposition 4.2}
\newtheorem*{prop_4.3}{Proposition 4.3}
\newtheorem*{prop_4.4}{Proposition 4.4}
\theoremstyle{definition}
\theoremstyle{remark}


\subsection{Proof of Proposition \ref{prop:contraction_target}: Contraction Mapping}
\label{proof:prop_convergence}
\begin{lemma}[Non-expansiveness of the Max Operator]
\label{lemma:dist_nonexpansive}
Let $Q(\mathbf{s}, \mathbf{a})$ be a scalar value function. The max operator over the action space is non-expansive in the $L_\infty$ norm. Specifically, for any two scalar value functions $Q_1$ and $Q_2$:
\[
\left\| \max_{\mathbf{a}} Q_1(\cdot, \mathbf{a}) - \max_{\mathbf{a}} Q_2(\cdot, \mathbf{a}) \right\|_\infty \le \| Q_1 - Q_2 \|_\infty.
\]
\end{lemma}
\begin{proof}
For any state $\mathbf{s}$, let $\mathbf{a}_1^* = \arg\max_{\mathbf{a}} Q_1(\mathbf{s}, \mathbf{a})$ and $\mathbf{a}_2^* = \arg\max_{\mathbf{a}} Q_2(\mathbf{s}, \mathbf{a})$. Without loss of generality, assume $Q_1(\mathbf{s}, \mathbf{a}_1^*) \ge Q_2(\mathbf{s}, \mathbf{a}_2^*)$. Then:
\[
\begin{aligned}
\max_{\mathbf{a}} Q_1(\mathbf{s}, \mathbf{a}) - \max_{\mathbf{a}} Q_2(\mathbf{s}, \mathbf{a}) &= Q_1(\mathbf{s}, \mathbf{a}_1^*) - Q_2(\mathbf{s}, \mathbf{a}_2^*) \\
&\le Q_1(\mathbf{s}, \mathbf{a}_1^*) - Q_2(\mathbf{s}, \mathbf{a}_1^*) \\
&\le \max_{\mathbf{a}} |Q_1(\mathbf{s}, \mathbf{a}) - Q_2(\mathbf{s}, \mathbf{a})| \\
&= \|Q_1 - Q_2\|_\infty.
\end{aligned}
\]
The case where $Q_2(\mathbf{s}, \mathbf{a}_2^*) > Q_1(\mathbf{s}, \mathbf{a}_1^*)$ follows by symmetry. Thus, the operation is non-expansive.

\end{proof}

\begin{prop_4.3}[Contraction Mapping under Target-Anchored Shaping]
Consider the reshaped MDP $\widetilde{\mathcal{M}} = \left< \mathcal{S}, \mathcal{A}, \widetilde{r}, \mathcal{P}, \gamma \right>$ with the reward dynamically modified by the target-anchored potential function. Let $\Phi(\mathbf{s}) = \eta \max_{\mathbf{a}} Q_{\bar{\theta}}(\mathbf{s}, \mathbf{a})$, where $Q_{\bar{\theta}}$ is a target network held fixed during the current planning and optimization step. The reshaped reward is defined as:
\[
\widetilde{r}_{\Phi}(\mathbf{s}, \mathbf{a}) = r(\mathbf{s}, \mathbf{a}) + \gamma \mathbb{E}_{\mathbf{s}' \sim \mathcal{P}}[\Phi(\mathbf{s}')] - \Phi(\mathbf{s})
\]
For any scaling factor $\eta > 0$, the Bellman Operator $\mathcal{T}_{\Phi}$ applied to this target-anchored reshaped MDP is a strict $\gamma$-contraction in the $L_\infty$ norm.
\end{prop_4.3}

\begin{proof}
Let $Q_1$ and $Q_2$ be two arbitrary scalar online $Q$-functions being optimized. The target-anchored Bellman operator $\mathcal{T}_{\Phi}$ is defined as:
\[
\mathcal{T}_{\Phi} Q_i(\mathbf{s}, \mathbf{a}) = \widetilde{r}_{\Phi}(\mathbf{s}, \mathbf{a}) + \gamma \mathbb{E}_{\mathbf{s}' \sim \mathcal{P}} \left[ \max_{\mathbf{a}'} Q_i(\mathbf{s}', \mathbf{a}') \right] \quad \text{for } i \in \{1, 2\}
\]

\textbf{Step 1: Independence of the Reshaped Reward.}
Crucially, because the potential function $\Phi$ is derived exclusively from the target network $Q_{\bar{\theta}}$ (which is held fixed during the current optimization step), it is completely independent of the online networks $Q_1$ and $Q_2$. Consequently, the reshaped reward $\widetilde{r}_{\Phi}(\mathbf{s}, \mathbf{a})$ acts as a stationary scalar for both evaluations, yielding a zero difference:
\[
\Delta \widetilde{r} = \widetilde{r}_{\Phi}(\mathbf{s}, \mathbf{a})\big|_{Q=Q_1} - \widetilde{r}_{\Phi}(\mathbf{s}, \mathbf{a})\big|_{Q=Q_2} = 0
\]
This isolates the contraction analysis to the discounted future value terms.

\textbf{Step 2: Bounding the Bellman Update Difference.}
We now evaluate the absolute difference between the Bellman updates for $Q_1$ and $Q_2$ at any state-action pair $(\mathbf{s}, \mathbf{a})$. By leveraging the linearity of expectation, the triangle inequality for integrals, and the non-expansiveness of the max operator (Lemma \ref{lemma:dist_nonexpansive}), we establish the following continuous bound:
\[
\begin{aligned}
\left| \mathcal{T}_{\Phi} Q_1(\mathbf{s}, \mathbf{a}) - \mathcal{T}_{\Phi} Q_2(\mathbf{s}, \mathbf{a}) \right| 
&= \bigg| \left( \widetilde{r}_{\Phi}(\mathbf{s}, \mathbf{a}) + \gamma \mathbb{E}_{\mathbf{s}'} \left[ \max_{\mathbf{a}'} Q_1(\mathbf{s}', \mathbf{a}') \right] \right) \\
&\qquad - \left( \widetilde{r}_{\Phi}(\mathbf{s}, \mathbf{a}) + \gamma \mathbb{E}_{\mathbf{s}'} \left[ \max_{\mathbf{a}'} Q_2(\mathbf{s}', \mathbf{a}') \right] \right) \bigg| \\
&= \gamma \left| \mathbb{E}_{\mathbf{s}'} \left[ \max_{\mathbf{a}'} Q_1(\mathbf{s}', \mathbf{a}') - \max_{\mathbf{a}'} Q_2(\mathbf{s}', \mathbf{a}') \right] \right| \\
&\le \gamma \mathbb{E}_{\mathbf{s}'} \left[ \left| \max_{\mathbf{a}'} Q_1(\mathbf{s}', \mathbf{a}') - \max_{\mathbf{a}'} Q_2(\mathbf{s}', \mathbf{a}') \right| \right] \\
&\le \gamma \mathbb{E}_{\mathbf{s}'} \left[ \|Q_1 - Q_2\|_\infty \right] \\
&= \gamma \|Q_1 - Q_2\|_\infty
\end{aligned}
\]
Taking the supremum over all $(\mathbf{s}, \mathbf{a})$, we obtain the global contraction bound:
\[
\|\mathcal{T}_{\Phi} Q_1 - \mathcal{T}_{\Phi} Q_2\|_\infty \le \gamma \|Q_1 - Q_2\|_\infty
\]
Since the discount factor $\gamma \in [0, 1)$, the operator $\mathcal{T}_{\Phi}$ is a strict $\gamma$-contraction mapping under the $L_\infty$ norm, regardless of the scaling factor $\eta$. By Banach's Fixed-Point Theorem, repeated applications of $\mathcal{T}_{\Phi}$ will converge to a unique fixed point for the current target-anchored landscape.

\end{proof}

\subsection{Proof of Proposition \ref{prop:policy_improvement}: Consistent Policy Improvement}
\label{proof:prop_improvement}
\begin{prop_4.4}[Consistent Policy Improvement]
Let $\pi_{k+1}$ be derived greedily from the reshaped landscape, i.e., $\pi_{k+1}(\mathbf{s}) = \arg\max_{\mathbf{a}} Q_k^{\pi_k}(\mathbf{s}, \mathbf{a})$. Then, this update monotonically improves the policy not only in the reshaped MDP $\widetilde{\mathcal{M}}_k$ but also with respect to the original objective $\mathcal{M}$, such that $Q_{\mathcal{M}}^{\pi_{k+1}}(\mathbf{s}, \mathbf{a}) \ge Q_{\mathcal{M}}^{\pi_k}(\mathbf{s}, \mathbf{a})$.
\end{prop_4.4}

\begin{proof}
Let $\mathcal{T}_k^{\pi_k}$ denote the Bellman evaluation operator for policy $\pi_k$ in the reshaped MDP snapshot $\widetilde{\mathcal{M}}_k$. By definition, the action-value function $Q_k^{\pi_k}$ is the fixed point of this operator:
\[
Q_k^{\pi_k}(\mathbf{s}, \mathbf{a}) = \mathcal{T}_k^{\pi_k} Q_k^{\pi_k}(\mathbf{s}, \mathbf{a}) = \widetilde{r}_k(\mathbf{s}, \mathbf{a}) + \gamma \mathbb{E}_{\mathbf{s}' \sim \mathcal{P}} \left[ Q_k^{\pi_k}(\mathbf{s}', \pi_k(\mathbf{s}')) \right]
\]

\textbf{Step 1: Local Policy Improvement in $\widetilde{\mathcal{M}}_k$.}
Consider a new policy $\pi_{k+1}$ that is perfectly greedy with respect to $Q_k^{\pi_k}$. By the definition of the max operator, for any state $\mathbf{s}'$, we have:
\[
Q_k^{\pi_k}(\mathbf{s}', \pi_{k+1}(\mathbf{s}')) = \max_{\mathbf{a}' \in \mathcal{A}} Q_k^{\pi_k}(\mathbf{s}', \mathbf{a}') \ge Q_k^{\pi_k}(\mathbf{s}', \pi_k(\mathbf{s}'))
\]
Applying the Bellman evaluation operator $\mathcal{T}_k^{\pi_{k+1}}$ to $Q_k^{\pi_k}$ and leveraging the monotonicity of the expectation, we obtain:
\[
\begin{aligned}
\mathcal{T}_k^{\pi_{k+1}} Q_k^{\pi_k}(\mathbf{s}, \mathbf{a}) 
&= \widetilde{r}_k(\mathbf{s}, \mathbf{a}) + \gamma \mathbb{E}_{\mathbf{s}' \sim \mathcal{P}} \left[ Q_k^{\pi_k}(\mathbf{s}', \pi_{k+1}(\mathbf{s}')) \right] \\
&\ge \widetilde{r}_k(\mathbf{s}, \mathbf{a}) + \gamma \mathbb{E}_{\mathbf{s}' \sim \mathcal{P}} \left[ Q_k^{\pi_k}(\mathbf{s}', \pi_k(\mathbf{s}')) \right] = Q_k^{\pi_k}(\mathbf{s}, \mathbf{a})
\end{aligned}
\]
By induction, repeated applications of the monotonic operator yield $(\mathcal{T}_k^{\pi_{k+1}})^n Q_k^{\pi_k} \ge Q_k^{\pi_k}$ for all $n \ge 1$. Since $\mathcal{T}_k^{\pi_{k+1}}$ is a strict $\gamma$-contraction in $L_\infty$ norm within the stationary snapshot $\widetilde{\mathcal{M}}_k$, the sequence converges to the unique fixed point $Q_k^{\pi_{k+1}}$, establishing the improvement in the reshaped landscape:
\[
Q_k^{\pi_{k+1}}(\mathbf{s}, \mathbf{a}) \ge Q_k^{\pi_k}(\mathbf{s}, \mathbf{a})
\]

\textbf{Step 2: Alignment with the Original Objective $\mathcal{M}$.}
To ensure this improvement is not a byproduct of shaping bias, we invoke the action-ordering invariance of PBRS. The relationship between the reshaped value and the original value is $Q_k^\pi(\mathbf{s}, \mathbf{a}) = Q_{\mathcal{M}}^\pi(\mathbf{s}, \mathbf{a}) - \Phi_k(\mathbf{s})$. In our framework, the potential function is instantiated as the state-value estimate from the target network:
\[
\Phi_k(\mathbf{s}) = \eta \max_{\mathbf{a}' \in \mathcal{A}} Q_{\bar{\theta}_k}(\mathbf{s}, \mathbf{a}').
\]
Crucially, although the underlying target network $Q_{\bar{\theta}_k}$ is a function of both state and action, the $\max$ operator identifies the supremum over the entire action manifold $\mathcal{A}$ for a given state $\mathbf{s}$. Consequently, the potential $\Phi_k(\mathbf{s})$ depends solely on the state topology and is strictly independent of the specific action $\mathbf{a}$ being evaluated in the current Bellman update.

The magnitude of value improvement between any two policies $\pi_{k+1}$ and $\pi_k$ is preserved exactly:
\[
Q_k^{\pi_{k+1}}(\mathbf{s}, \mathbf{a}) - Q_k^{\pi_k}(\mathbf{s}, \mathbf{a}) = Q_{\mathcal{M}}^{\pi_{k+1}}(\mathbf{s}, \mathbf{a}) - Q_{\mathcal{M}}^{\pi_k}(\mathbf{s}, \mathbf{a})
\]
Thus, the non-negative improvement proven in Step 1 directly implies $Q_{\mathcal{M}}^{\pi_{k+1}}(\mathbf{s}, \mathbf{a}) \ge Q_{\mathcal{M}}^{\pi_k}(\mathbf{s}, \mathbf{a})$.

\end{proof}

\clearpage

\section{Implementation Details}

\subsection{Hyperparameters}
The majority of hyperparameters in our experiments are inherited from baseline algorithm TD-MPC2 and DEMO$^3$, to ensure a fair comparison and reproducibility. For consistency, we retain the default settings for most optimization, model, and policy-related parameters as adopted in previous works. In this section, we highlight only the hyperparameters that are most pertinent to our proposed method, such as those related to reward shaping, quantile selection. A detailed listing of these key hyperparameters is provided in Table~\ref{tab:hyperparams}. For all remaining parameters and further implementation details, we refer readers to the original papers of TD-MPC2. 

\begin{table}[h]
\centering
\vskip -0.05in
\caption{SLOPE's hyperparameters.}
\setlength{\tabcolsep}{10pt} 
\begin{tabular}{ll}
\midrule 
\textbf{Hyperparameter} & \textbf{Value} \\ 
\midrule 
\rowcolor{bluebg}
\textbf{\textcolor{blue}{\underline{SLOPE}}}           &                \\
\rowcolor{bluebg}
Potential weight $\eta$       & 1.0           \\
\rowcolor{bluebg}
Quantile coefficient $\tau$         & 0.55            \\ 
\midrule 
\textbf{\textcolor{blue}{\underline{TD-MPC2}}}      &                \\
Planning horizon $H$      & 3              \\
MPPI iterations      & Adroit: 8, others: 6              \\
Interaction steps      & 500,000              \\
Replay buffer capacity    &300,000    \\
Sampling                  &Uniform  \\
Demo. sampling ratio    &50\%   \\
Eval episodes      & 10              \\
Num of parallel envs      & 16              \\
Steps per update      & 4              \\
Reward prediction coef. &0.1   \\
Value prediction coef. &0.1   \\
Joint-embedding coef.  &20   \\
Temporal coef. $\lambda$ &0.5   \\
Latent state dim        & 512            \\
Batch size              & 256            \\
Optimizer               & Adam           \\
Learning rate           & $3\times 10^{-4}$              \\
Encoder learning rate   & $1\times 10^{-4}$              \\
Discount factor $\gamma$         & 0.95          \\ 
\midrule 
\end{tabular}
\label{tab:hyperparams}
\end{table}

\subsection{Environment Details} 
In this section, we provide a comprehensive overview of the four benchmark environments utilized in our experiments: \textit{ManiSkill3}, \textit{Meta-World}, \textit{Adroit}, \textit{Robosuite}, and \textit{Real World} experiments. The key characteristics and configurations of these environments are summarized in the following Table \ref{table:env_details}.

\begin{table}[H]
\centering
\vskip -0.05in
\caption{Summary of environment details for \textit{ManiSkill3}, \textit{Meta-World}, \textit{Adroit}, \textit{Robosuite}, and \textit{Real World} experiments. Prop.state stands for proprioceptive state of the robotic arm.}
\resizebox{\linewidth}{!}{
\begin{tabular}{l|ccccc}
\midrule 
                           & \textbf{ManiSkill3}        & \textbf{Meta-World} & \textbf{Adroit}            & \textbf{Robosuite} & \textbf{Real World} \\ 
                           \midrule 
\textbf{Time Horizon}      & 100                        & 100                 & 100                        & 100                & Variable            \\
\textbf{Image Size}        & 128$\times$128             & 224$\times$224      & 224$\times$224             & 128$\times$128     & 128$\times$128      \\
\textbf{Observations}      & RGB$\times$2 + Prop.state  & RGB + Prop.state    & RGB$\times$2 + Prop.state  & RGB$\times$2       & RGB$\times$2 + Prop.state        \\
\textbf{Cameras}           & Hand + Front               & Front               & Front                      & Hand + Front       & Hand + Side        \\
\textbf{Action Repeat} & 2         & 2                   & 2                          & 1                  & 1                   \\
\textbf{Action Dim}        & 7                          & 4                   & Door: 28, Hammer: 26       & 7                  & 4                   \\ \hline
\end{tabular}
}
\label{table:env_details}
\vskip -0.1in
\end{table}

\subsection{Demonstrations}
The demonstration data for all simulation environments used in our experiments are collected by training a TD-MPC2 agent with access to dense rewards and full state observations. The real-world demonstration data was collected by humans using a SpaceMouse remote control. Each demonstration consists of a fixed-length trajectory generated by the expert policy. The specific number of demonstrations employed for each task is summarized in Table~\ref{tab:demo-num}.

\begin{table}[htbp]
\centering
\caption{Number of demonstration trajectories used for each task in our experiments.}
\begin{tabular}{llc}
\midrule 
\textbf{Domain}                      & \textbf{Task}             & \textbf{Number of Demonstrations} \\ 
\midrule 
\multirow{5}{*}{ManiSkill3} & Peg Insertion    & 100                      \\
                            & Pick Place       & 100                      \\
                            & Stack Cube       & 25                       \\
                            & Poke Cube        & 5                        \\
                            & Lift Peg Upright & 5                        \\ \hline
\multirow{10}{*}{Meta-World} & Assembly         & 5                        \\
                            & Peg Insert Side  & 5                        \\
                            & Stick Push       & 5                        \\
                            & Stick Pull       & 5                        \\
                            & Hand Insert       & 5                        \\
                            & Pick Out of Hole       & 5                        \\ 
                            & Pick Place Wall       & 5                        \\  
                            & Push Back       & 5                        \\ 
                            & Push       & 5                        \\ 
                            & Shelf Place       & 5                        \\ \hline
\multirow{2}{*}{Adroit}     & Door             & 5                        \\
                            & Hammer           & 5                        \\ \hline
\multirow{3}{*}{Robosuite}  & Lift             & 5                        \\
                            & Door             & 10                       \\
                            & Stack Blocks     & 100                      \\ \hline
\multirow{3}{*}{Real-World Experiments}  & Press Button             & 1                        \\
                            & Push Cube             & 1                       \\
                            & Grasp Cube     & 10                      \\ 
                            \midrule 
\end{tabular}
\label{tab:demo-num}
\end{table}

\subsection{Real-Robot Implementation Details}
\label{appendix:real_robot}

\paragraph{Hardware Setup.} As shown in Fig.\ref{fig:hardware}, we use a Realman RM75-B 7-DoF robotic arm equipped with a gripper. The visual perception system consists of two RGB cameras: one is mounted on the wrist (eye-in-hand) to provide egocentric views, while the other is fixed at the side to offer a third-person perspective. Both cameras capture RGB images, which are subsequently cropped to a resolution of $128 \times 128$.

\begin{wrapfigure}[12]{r}{0.30\textwidth}
\centering
\vskip -0.2in
\includegraphics[width=0.30\textwidth]{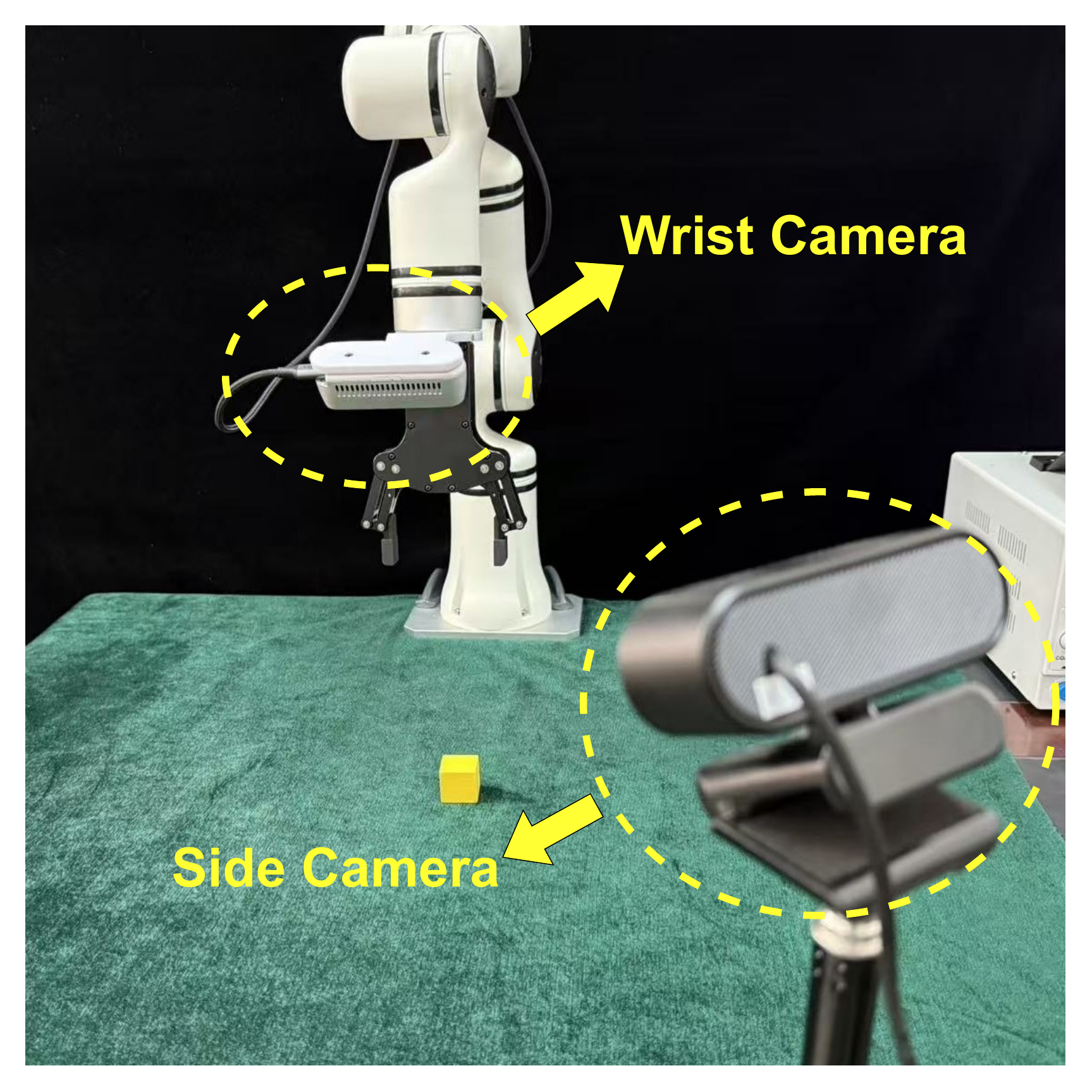}
\caption{Real-robot setup.}
\vskip -0.2in
\label{fig:hardware}
\end{wrapfigure}

\paragraph{Action Space and Controller.} The policy outputs a 4-dimensional action vector $a_t \in \mathbb{R}^4$, consisting of the end-effector's relative displacement $(\Delta x, \Delta y, \Delta z)$ and the gripper status.

\begin{itemize}
    \item \textbf{Remote Control:} We employ a 3Dconnexion SpaceMouse for teleoperation. The human operator controls the end-effector's translational velocity.
    
    \item \textbf{Low-level Control:} The target Cartesian pose is converted to joint positions using an analytical Inverse Kinematics (IK) solver. The commands are sent to the robot at a frequency of 10Hz, with a safety layer that clips joint velocities and confines the end-effector within a predefined workspace boundaries.
\end{itemize}

\begin{itemize}
    \item  \textbf{Smoothed Action:} To ensure smooth robot motion and mitigate high-frequency jitter, we apply an Exponential Moving Average (EMA) filter to the policy's output actions. Let $a_t = [a_{t, \text{trans}}, a_{t, \text{grip}}]$ denote the raw action at time $t$, where $a_{t, \text{trans}} \in \mathbb{R}^3$ represents the translational displacement and $a_{t, \text{grip}} \in \mathbb{R}^1$ is the gripper command. The smoothed action $\bar{a}_t$ is computed as:
    \[
    \bar{a}_{t, \text{trans}} = \alpha \cdot a_{t, \text{trans}} + (1-\alpha) \cdot \bar{a}_{t-1, \text{trans}}, \quad \bar{a}_{t, \text{grip}} = a_{t, \text{grip}}
    \]
    where $\alpha \in [0, 1]$ is the smoothing factor. This formulation dampens control noise in the trajectory while preserving the immediate responsiveness of gripper actuation.
\end{itemize}

\paragraph{Hybrid Action Space.} The robotic manipulation task involves a hybrid action space $\mathcal{A} = \mathcal{A}_{\text{arm}} \times \mathcal{A}_{\text{grip}}$, where the arm motion $\mathcal{A}_{\text{arm}} \subseteq \mathbb{R}^3$ is continuous, but the gripper actuation $\mathcal{A}_{\text{grip}} \in \{-1, 1\}$ is discrete (binary). Standard MPPI assumes a fully continuous space with Gaussian noise, which is ill-suited for binary controls. We implement a specific strategy to bridge this gap:

\begin{itemize}
        \item \textbf{Policy Discretization via STE:} To enable end-to-end training of the policy network $\pi_\theta$ despite the non-differentiable binary output, we utilize the Straight-Through Estimator (STE). During the forward pass, the gripper action is binarized using a hard threshold function $a_{\text{grip}} = \text{sign}(a_{\text{raw}})$. During the backward pass, gradients are propagated through the continuous $a_{\text{raw}}$ (using the identity function for the derivative), effectively treating the threshold as a transparent layer.
        
        \item \textbf{Quantized MPPI Sampling:} During the planning phase, we modify the sampling procedure of the Model Predictive Path Integral (MPPI). While we sample candidate trajectories from a Gaussian distribution $\mathcal{N}(\mu, \Sigma)$, we enforce a hard quantization on the gripper dimension immediately after sampling:
        \[
        a_{t, \text{grip}}^{(i)} \leftarrow \text{sign}(a_{t, \text{grip}}^{(i)}), \quad \forall i \in \{1, \dots, N\}
        \]
        This ensures that the latent dynamics model predicts future states conditional on valid discrete gripper commands, reducing prediction error caused by out-of-distribution continuous values (e.g., $0.0$).
        
        \item \textbf{Distribution Projection:} When updating the sampling distribution parameters via the Cross-Entropy Method (CEM), we project the updated mean of the gripper dimension back to the binary manifold. Specifically, after computing the weighted mean $\mu_{\text{new}}$ from elite trajectories, we apply $\mu_{\text{new}, \text{grip}} \leftarrow \text{sign}(\mu_{\text{new}, \text{grip}})$. This prevents the sampling distribution from collapsing towards zero (ambiguous state) and maintains distinct modes for ``open" and ``close" actions throughout the planning horizon.
\end{itemize}

\paragraph{Task Settings.}
We design three sparse-reward manipulation tasks of increasing complexity to evaluate the efficacy of our method in real-world scenarios.

\begin{minipage}[c]{0.86\textwidth}
\begin{itemize}
    \item \texttt{Press Button(Easy)}: The goal is to control the end-effector to press a specific button on the table. It is a \textit{sparse-reward} task where the agent receives a success signal ($r=1$) only when the button is fully depressed. This task requires precise end-effector positioning and force application.
\end{itemize}
\end{minipage}
\hfill
\begin{minipage}[c]{0.1\textwidth} 
    \centering
    \includegraphics[width=\linewidth]{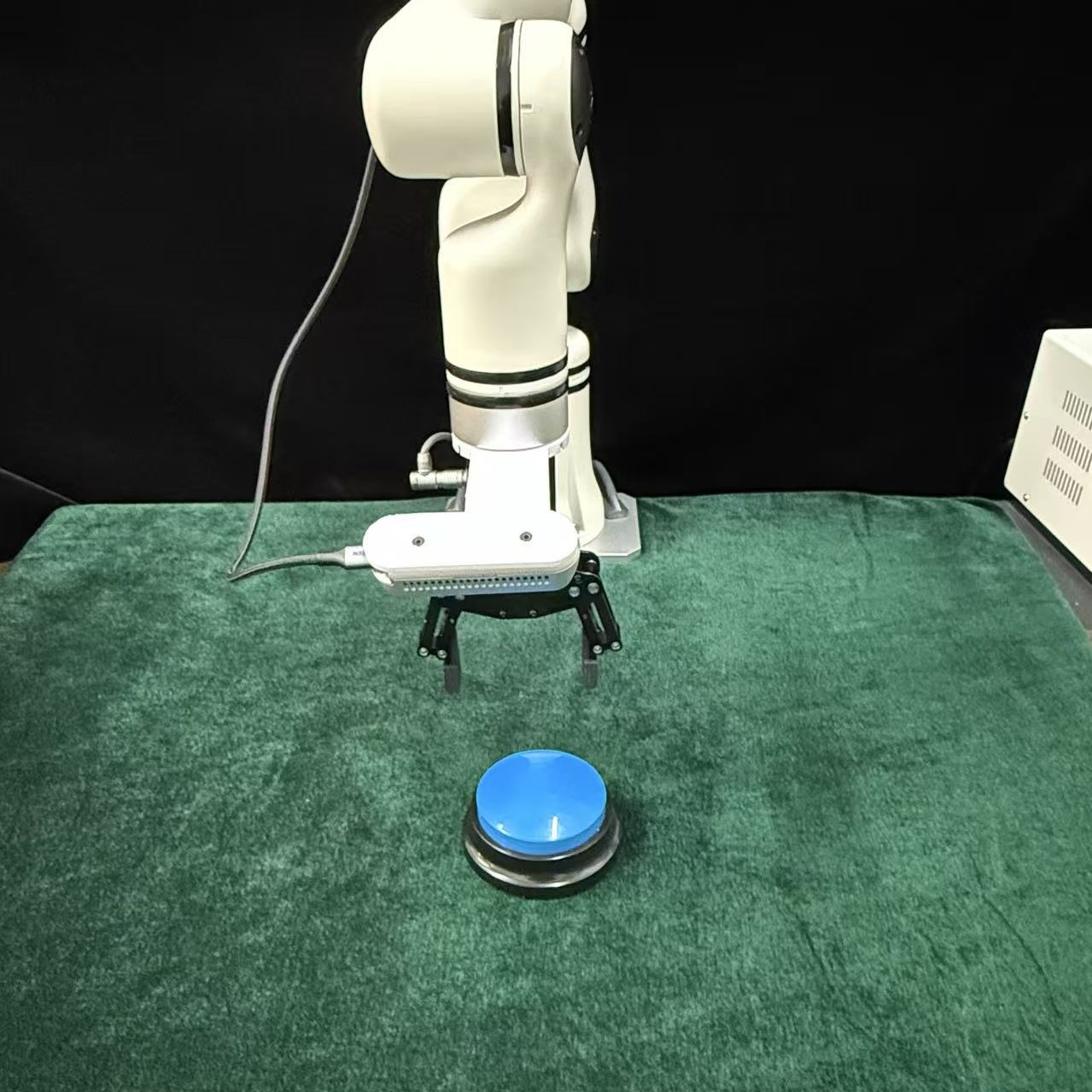} 
\end{minipage}

\begin{minipage}[c]{0.86\textwidth}
\begin{itemize}
    \item \texttt{Push Cube(Medium)}: The agent must push a target cube from a randomized starting position to a designated goal region marked on the table. The reward is binary, provided only when the object successfully enters the target zone. This task challenges the agent's ability to reason about object dynamics and contact physics.
\end{itemize}
\end{minipage}
\hfill
\begin{minipage}[c]{0.1\textwidth}
    \centering
    \includegraphics[width=\linewidth]{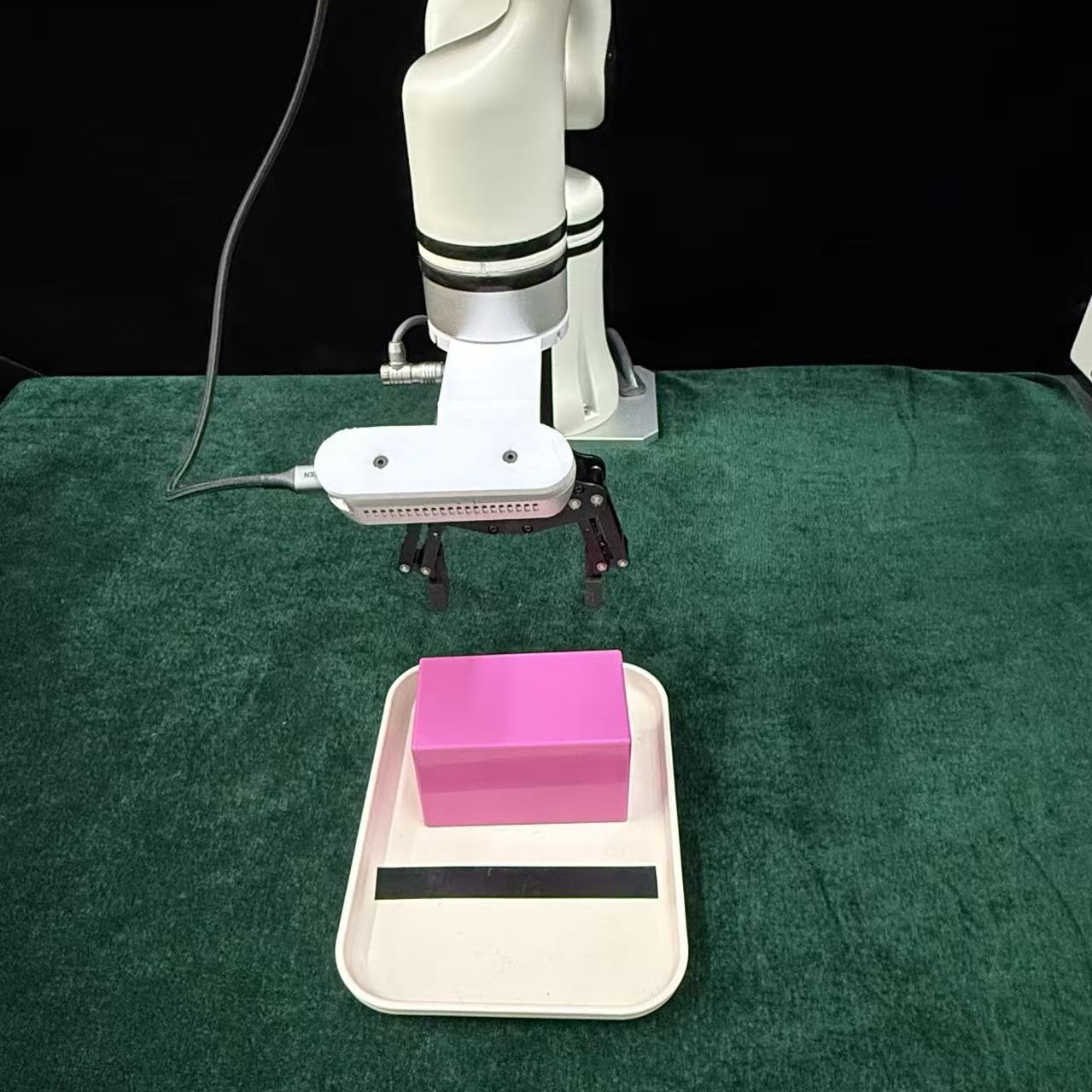}
\end{minipage}

\begin{minipage}[c]{0.86\textwidth}
\begin{itemize}
    \item \texttt{Grasp Cube(Hard)}: The objective is to pick up a cube from the table. A success reward is granted only when the object is lifted to a predefined height threshold. This is the most challenging task as it involves precise grasping, stability control during lifting, and handling potential slippage.
\end{itemize}
\end{minipage}
\hfill
\begin{minipage}[c]{0.1\textwidth}
    \centering
    \includegraphics[width=\linewidth]{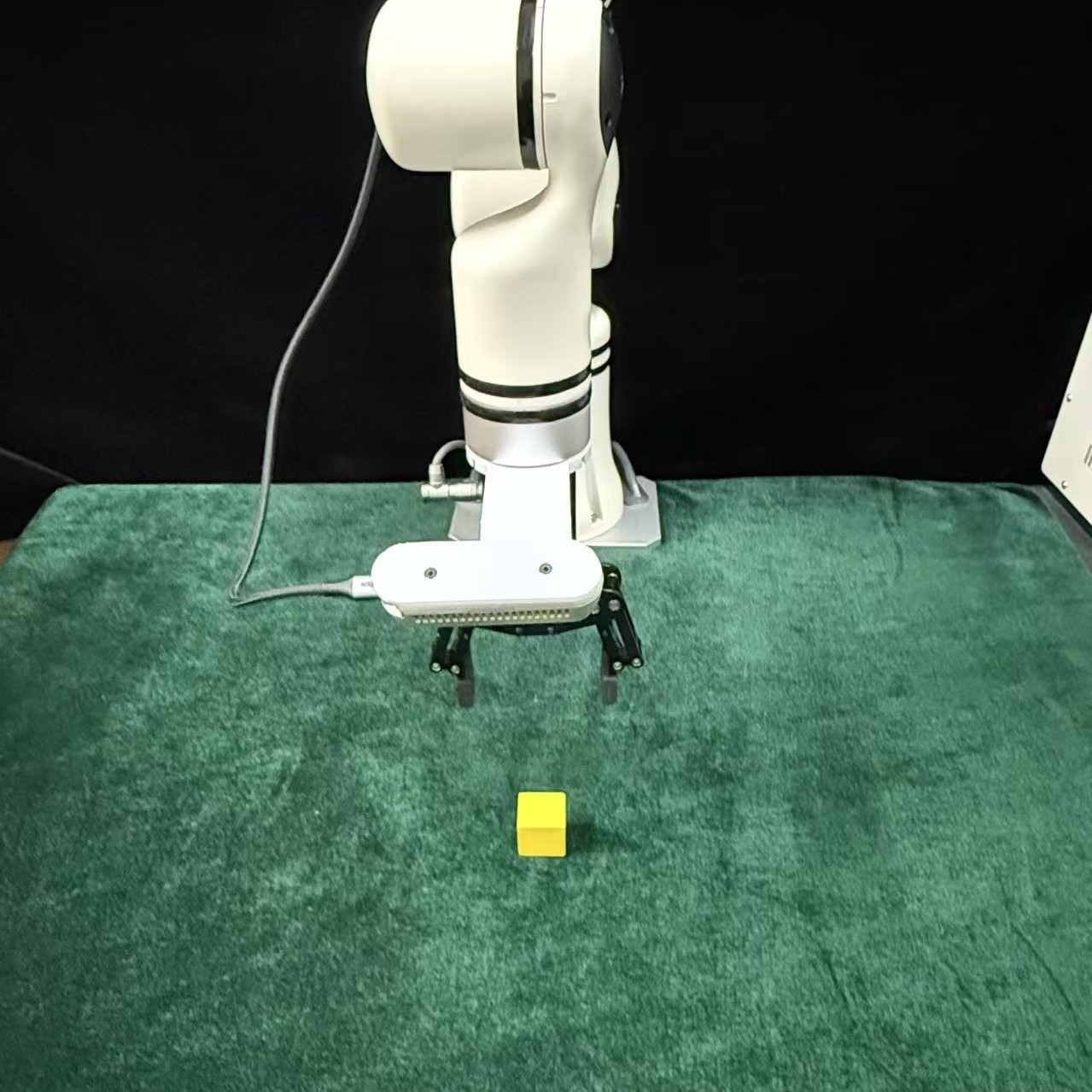}
\end{minipage}

\paragraph{Reset Mechanism.} At the beginning of each episode, the robotic arm moves to a fixed reset pose. Subsequently, the task scene is manually restored by a human operator. To ensure the robustness of the learned policy, the operator randomizes the object's initial position and orientation within a predefined workspace during each reset.

\paragraph{Reward Function.} Due to the absence of ground-truth state information in the real world, we rely on sparse reward signals to indicate task completion. Generally, there are two common approaches to obtain these signals: training a visual binary classifier to predict success probabilities, or employing manual human labeling. In this work, we adopt the latter approach, where a human operator verifies the task success status during policy execution.

\paragraph{Training Process.} We adopt the HIL-SERL \cite{hil-serl} paradigm for real-world training, implemented based on the official LeRobot codebase\footnote{\url{https://huggingface.co/docs/lerobot/hilserl}}. A key component of HIL-SERL is the active human intervention mechanism. During the agent's autonomous rollout, a human operator monitors the execution and can instantaneously intervene via the SpaceMouse to correct deviations or guide the agent out of unpromising states, thereby facilitating efficient exploration. 

Instead of a single replay buffer, HIL-SERL maintain two distinct replay buffers to manage the data effectively: an \textit{online RL buffer} that stores transitions from the agent's autonomous policies, and a \textit{demonstration buffer} that houses pre-collected expert demonstrations and human intervention data. For policy updates, we construct training batches by sampling from both buffers simultaneously with a \textbf{1:1} ratio. This balanced sampling strategy allows the agent to learn robustly from a diverse distribution of expert guidance and self-generated experiences.

\paragraph{Hyperparameters.} We provide the detailed hyperparameters for the environment setup and algorithm training in Table \ref{tab:real_world_hyperparams}.

\begin{table}[h]
\centering
\caption{Hyperparameters for Real-Robot Experiments.}
\setlength{\tabcolsep}{12pt}
\begin{tabular}{ll}
\midrule 
\textbf{Hyperparameter} & \textbf{Value} \\ 
\midrule 
\textbf{\textcolor{blue}{\underline{Environment \& General}}} & \\
Image resolution & $128 \times 128$ \\
Proprioceptive state dim & 4 \\
Control frequency & 10 Hz \\
Action dimension & 4 \\
Batch size & 256 \\
Total steps & 15,000 \\ 
\midrule 
\textbf{\textcolor{blue}{\underline{SAC}}} & \\
Actor/Critic Learning rate & $3 \times 10^{-4}$ \\
Discount factor $\gamma$ & 0.97 \\
UTD ratio & 2 \\
Hidden dimension & 256 \\
Target smoothing $\tau$ & 0.005 \\
Online buffer capacity & 30,000 \\ 
\midrule 
\textbf{\textcolor{blue}{\underline{MoDem (SLOPE)}}} & \\
Planning horizon $H$ & 5 \\
CEM iterations & 6 \\
Q ensemble size & 5 \\
Latent dimension & 50 \\
MLP dimension & 512 \\
Discount factor $\gamma$ & 0.99 \\
Consistency coefficient & 20.0 \\
Value coefficient & 0.1 \\
Pi coefficient & 0.5 \\
Temporal decay coefficient & 0.5 \\
Target model momentum & 0.995 \\
Smoothing factor $\alpha$ & 0.4 \\
\midrule 
\end{tabular}
\label{tab:real_world_hyperparams}
\end{table}

\clearpage

\section{Additional Experimental Results}
\label{sec:additional_results}
\subsection{Visualization and Analysis of Reward Landscapes}
We visualize the reward landscapes of TD-MPC2, MoDem, and SLOPE separately during rollout. As shown in Fig.~\ref{fig:reward_landscape}, TD-MPC2 yields an almost flat reward landscape, indicating that its reward model fails to provide meaningful guidance for policy evaluation. Compared to MoDem, SLOPE generates a more discriminative reward landscape, while MoDem remains relatively flat. This suggests that SLOPE offers clearer optimization signals, enabling more effective credit assignment and policy improvement under sparse rewards.

\begin{figure}[h!]
  \centering
  \begin{subfigure}[b]{0.3\textwidth}
    \centering
    \includegraphics[width=\textwidth]{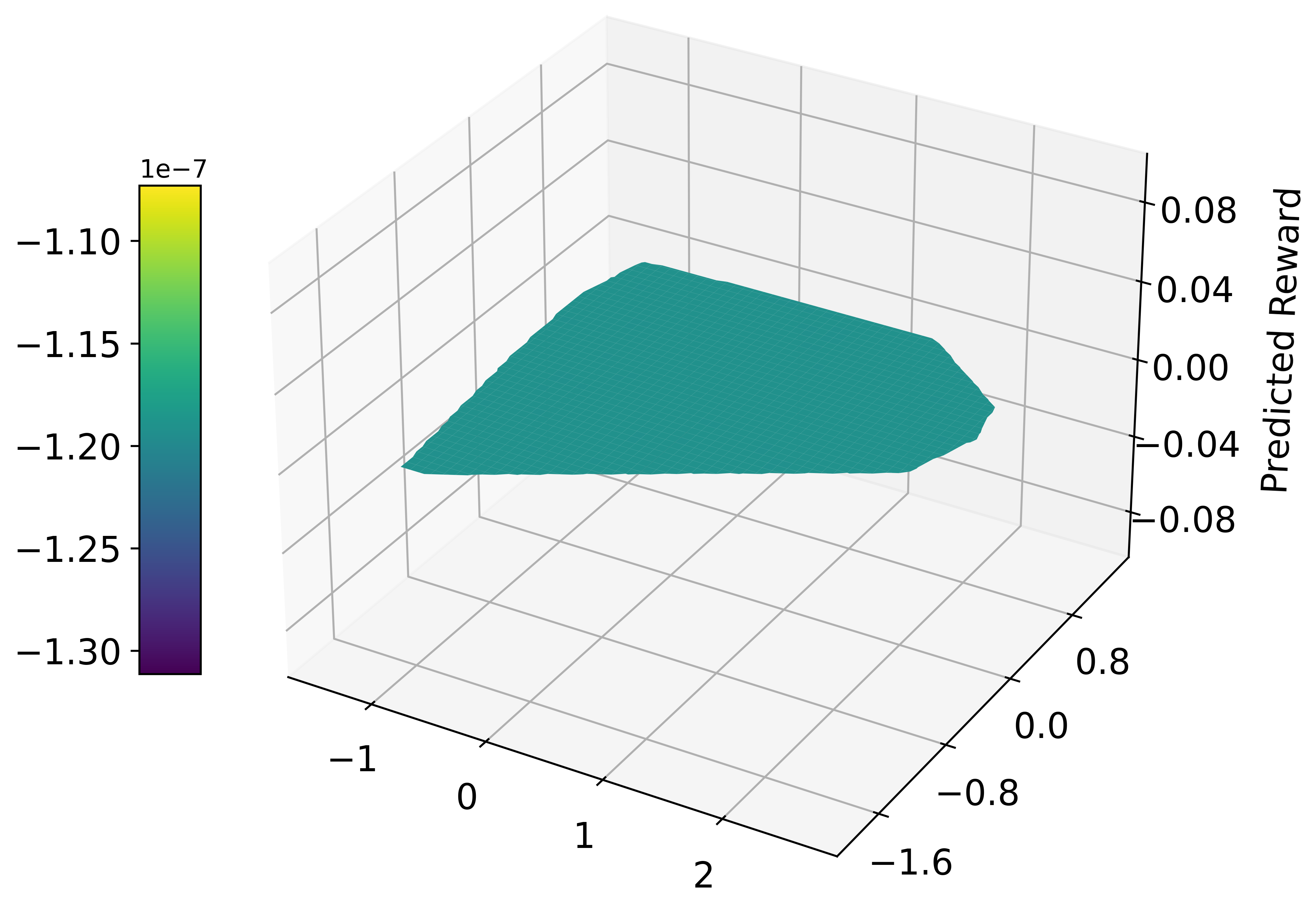}
    \caption{TD-MPC2}
    \label{fig:tdmpc2}
  \end{subfigure}
  \hfill
  \begin{subfigure}[b]{0.3\textwidth}
    \centering
    \includegraphics[width=\textwidth]{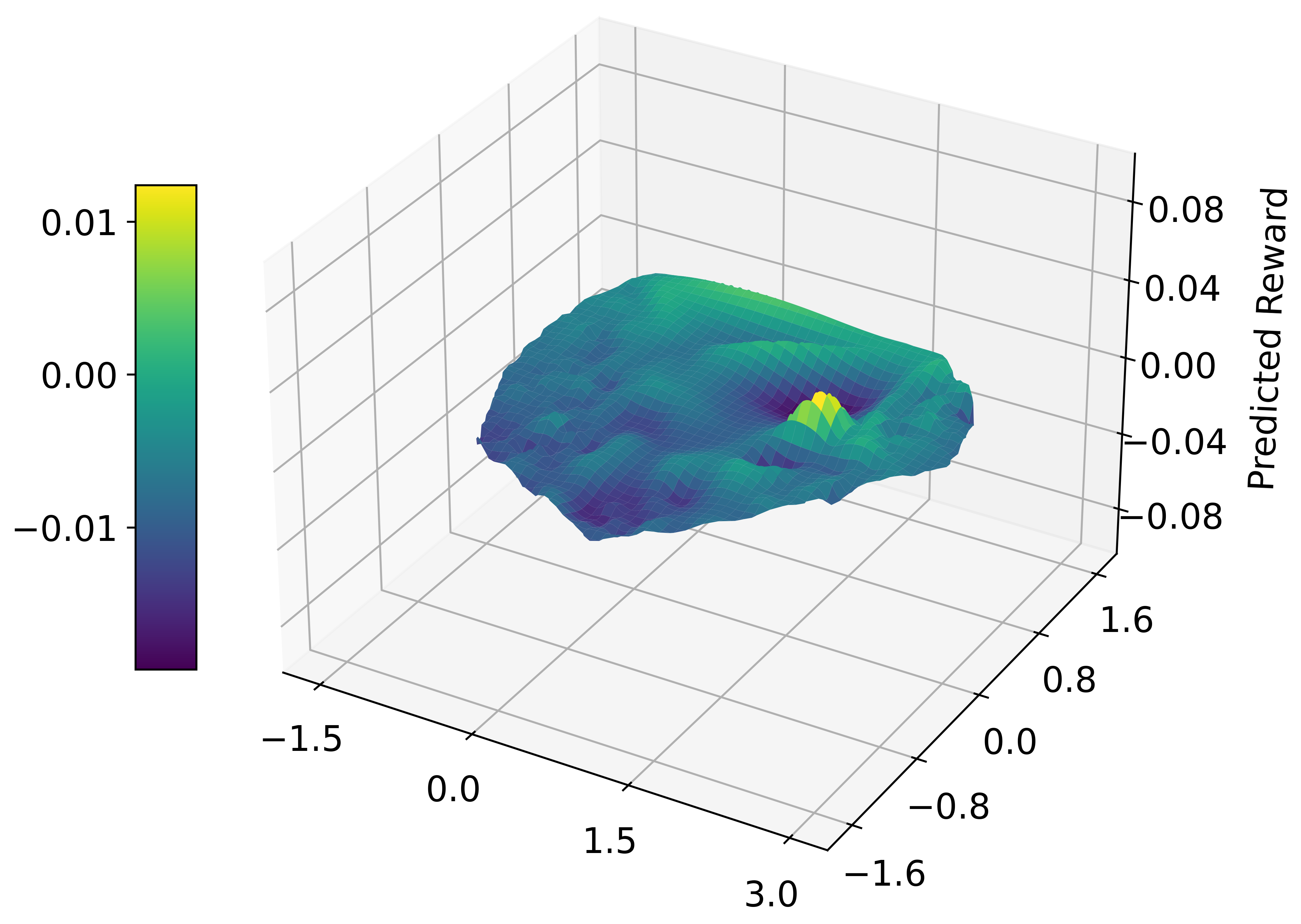}
    \caption{MoDem}
    \label{fig:modem}
  \end{subfigure}
  \hfill
  \begin{subfigure}[b]{0.3\textwidth}
    \centering
    \includegraphics[width=\textwidth]{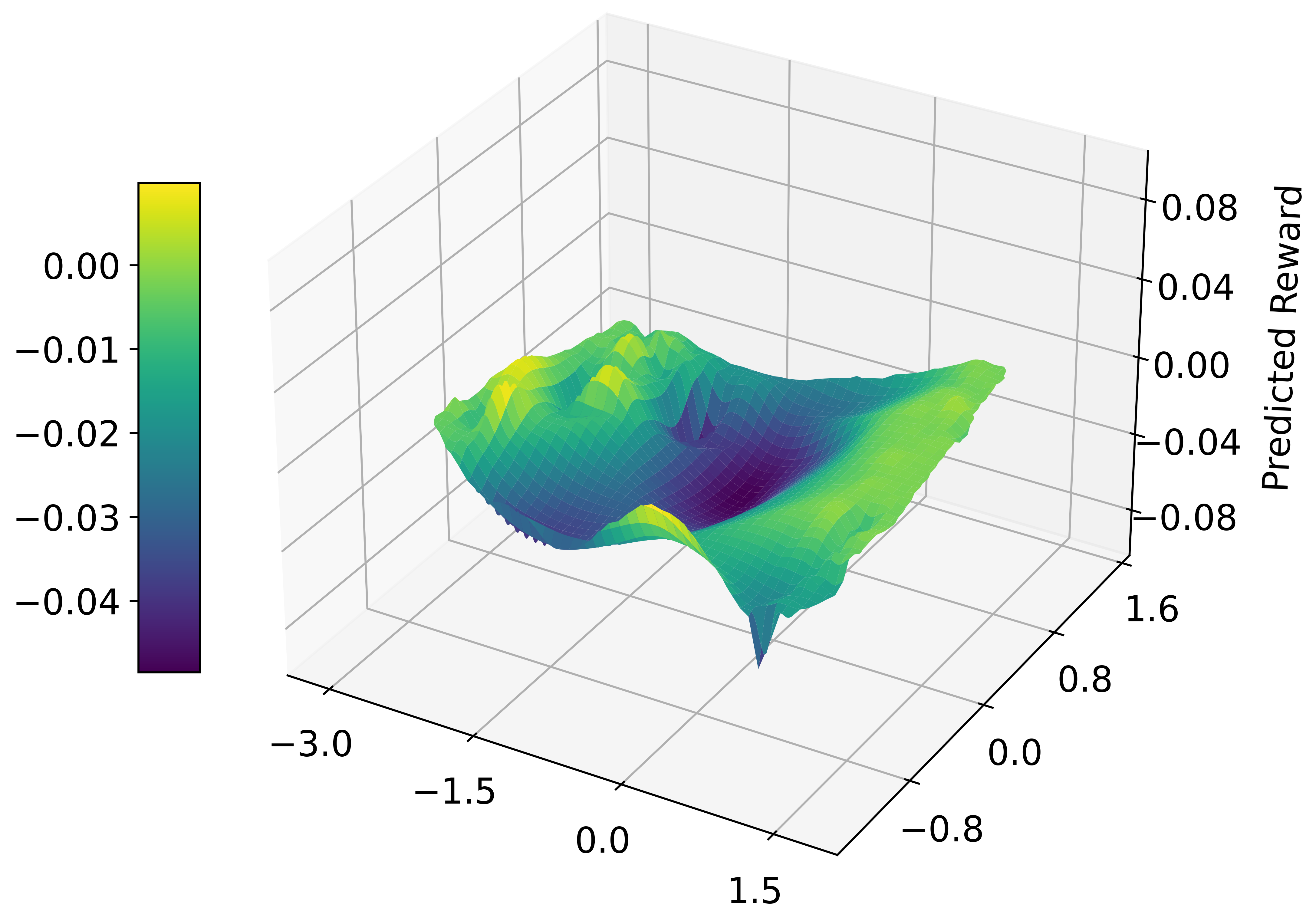}
    \caption{\textbf{SLOPE(Ours)}}
    \label{fig:slope}
  \end{subfigure}
  \caption{Visualization of reward landscape. We collect trajectory data during rollout as the basis for visualization. All methods are evaluated on the same set of states, using their respective reward models to predict rewards. The high-dimensional state representations are projected into a 2D space using Principal Component Analysis (PCA), and predicted rewards are plotted over a mesh grid to construct the reward surfaces. To avoid distortion from extreme values caused by successful states (which yield disproportionately high predicted rewards), we exclude them and visualize only states with ground-truth rewards of zero.}
  \label{fig:reward_landscape}
\end{figure}

\subsection{Visualization of Shaped Reward Evolution in GridWorld}

To further investigate the mechanism of our Bootstrapped Potential Construction, we visualize the temporal evolution of the shaped reward landscape during the training process on the $10 \times 10$ GridWorld task introduced in Section \ref{sec:pbrs}. Fig.\ref{fig:toy_example_iter} displays the heatmaps of the shaped reward $\widetilde{r}(\mathbf{s}, \mathbf{a})$ at varying training iterations (iteration 10, 20, 30, 40, and 100).

As shown in Fig.\ref{fig:toy_example_iter}, in the early training stage (iteration 10), high rewards are limited to the area around the goal state (bottom-right corner). The rest of the state space provides little information, similar to the original sparse reward setting. However, as training continues (from iteration 20 to 40), the reward signal propagates backward from the goal towards the start state. By iteration 100, the reward landscape extends across the entire state space. Meanwhile, the color transition becomes smoother, establishing a clear directional trend that effectively guides the agent towards the goal.

\begin{figure}[h!]
    \centering
    \includegraphics[width=1.0\linewidth]{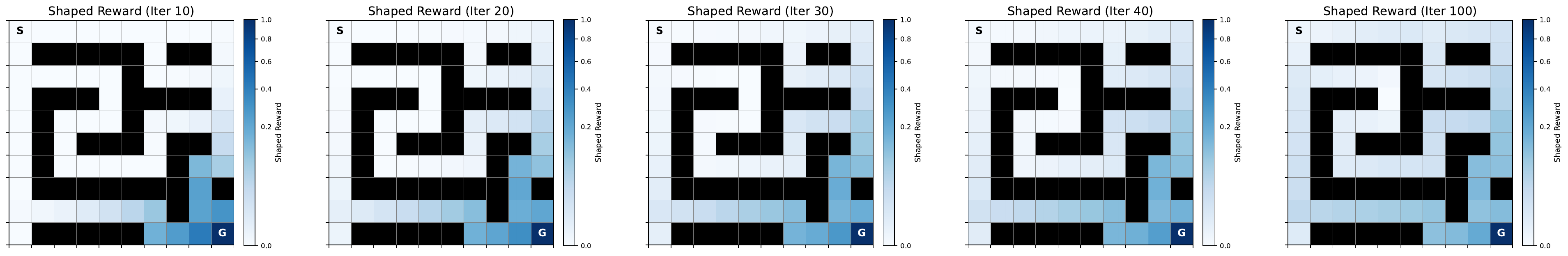}
    \caption{Visualization of the shaped reward landscape evolution on the $10 \times 10$ GridWorld. Darker blue regions indicate higher reward values. As training progresses, the dense reward signal propagates from the goal state (bottom-right) backwards to the initial state (top-left), progressively forming a global gradient field.}
    \label{fig:toy_example_iter}
\end{figure}

\subsection{Detailed Performance on Each Task}
\label{sec:detail_results}
Fig.\ref{fig:overall_results} visualizes the aggregate performance of SLOPE against baseline algorithms on all 20 evaluated tasks. This extensive evaluation covers distinct domains (\textit{ManiSkill3}, \textit{Meta-World}, \textit{RoboSuite}, \textit{Adroit}) and varying reward densities (sparse and semi-sparse). As evidenced by the curves, SLOPE not only converges faster but also exhibits lower variance across different seeds, highlighting its robustness and generalizability in solving diverse manipulation problems.

\begin{figure}[t]
    \centering
    \includegraphics[width=1.0\linewidth]{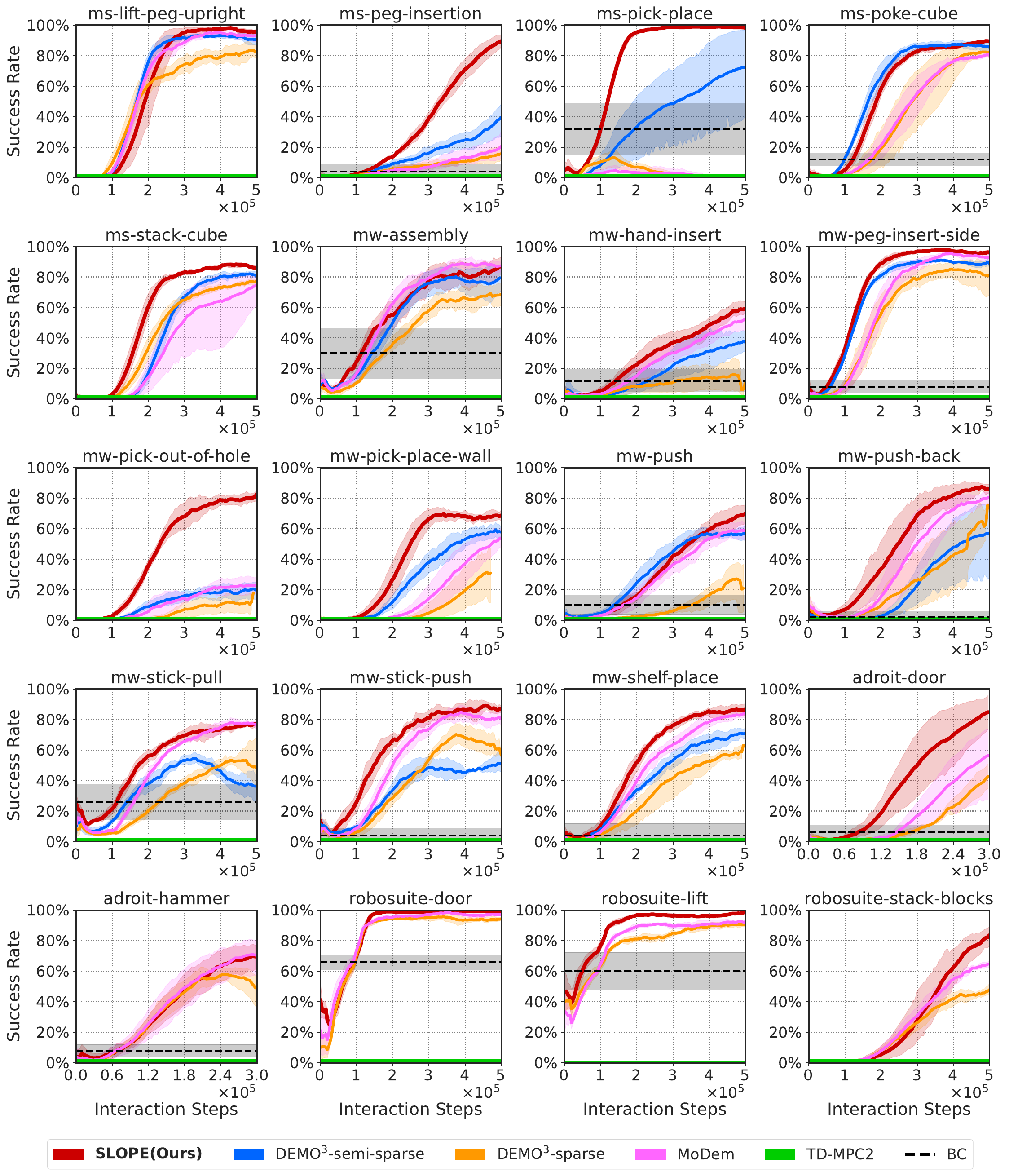}
    \caption{Success rate comparison on 20 tasks. We conduct all experiments using 5 distinct random seeds $\{1, 2, 3, 4, 5\}$.}
    \label{fig:overall_results}
\end{figure}

\clearpage

\subsection{Evaluation on Dense Reward Tasks}
While the primary contribution of SLOPE lies in resolving sparse reward challenges, we further investigate its efficacy in standard dense reward settings to assess its generality. We evaluate SLOPE on the DMC, selecting 8 representative locomotion tasks with state-based inputs. We compare our method against three baselines: \textbf{TD-MPC2} (our backbone algorithm), \textbf{Dreamerv3}, and \textbf{SAC}. The results of SAC, TD-MPC2 and Dreamerv3 are downloaded from this link \footnote{\url{https://github.com/nicklashansen/tdmpc2/tree/main/results}}.

\paragraph{Implementation Details for Dense Rewards.} 
It is worth noting that for these dense reward experiments, we modified the SLOPE configuration. Since the environment provides continuous and informative feedback, the exploration challenge is significantly mitigated compared to sparse settings. Therefore, we disabled the Optimism-Driven Landscape Shaping component and relied solely on the Bootstrapped Potential Construction. In this configuration, SLOPE essentially acts as TD-MPC2 augmented with a dynamic PBRS mechanism. 

\begin{figure}[h!]
    \centering
    \includegraphics[width=1.0\linewidth]{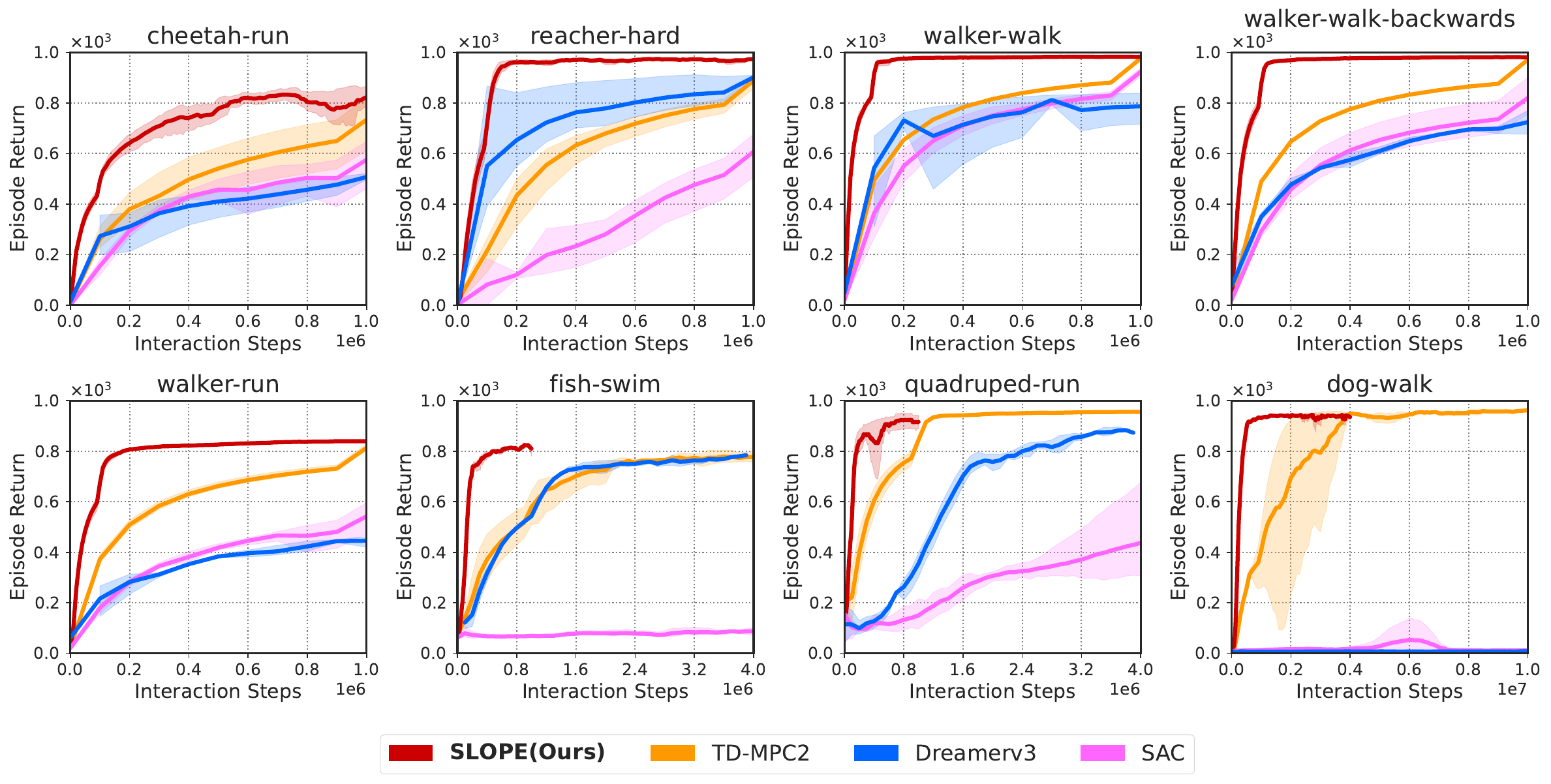}
    \caption{Performance comparison on 8 tasks from the DMC with dense rewards. We compare SLOPE (Ours) against TD-MPC2, Dreamerv3, and SAC. All results are averaged over 5 random seeds, with shaded regions indicating 95\% CIs. SLOPE consistently achieves faster convergence and higher asymptotic performance compared to the baselines, demonstrating that bootstrapped potential shaping accelerates learning even in dense reward settings.}
    \label{fig:dmc_dense}
\end{figure}

\paragraph{Results and Analysis.}
The results on the DMC tasks are illustrated in Fig.\ref{fig:dmc_dense}. First, compared to the backbone TD-MPC2 (orange curves), SLOPE (red curves) demonstrates significantly faster convergence speeds across nearly all tasks. This acceleration suggests that shaping the reward with the learned value function creates a more favorable optimization landscape, allowing the planner to identify high-reward trajectories more efficiently. Furthermore, in terms of asymptotic performance, SLOPE consistently matches or outperforms both the model-free baseline (SAC) and the model-based baseline (DreamerV3). These results collectively confirm that the benefits of Bootstrapped Potential Construction extend beyond sparse reward domains, acting as an effective regularizer that efficiently guides the planner within dense reward landscapes.

\subsection{Generalization to Other MBRL Algorithm: Dreamerv3}

To assess the universality of our framework beyond the specific architecture of TD-MPC2, we integrated SLOPE into Dreamerv3 \citep{dreamerv3}, a leading MBRL algorithm known for its robust performance across diverse domains.

\paragraph{Implementation Details.}
The integration of SLOPE into Dreamerv3 requires a distinct strategy due to the architectural differences in how reward models are utilized.
\begin{itemize}
    \item \textbf{Difference in Backbone:} TD-MPC2 employs the learned reward function primarily for inference-time planning (i.e., evaluating trajectory return via MPPI). In contrast, Dreamerv3 decouples representation learning from Actor/Critic learning. The world model learns to predict rewards to shape the latent representation, while the Actor and Critic learn behaviors via latent imagination.
    \item \textbf{Integration Strategy:} To preserve the fidelity of the world model's representation learning, we do not alter the training target of the reward predictor in the world model. The reward predictor continues to learn the ground-truth environmental reward $r(\mathbf{s}, \mathbf{a})$. Instead, we intervene during the Critic learning phase. Specifically, during the latent imagination rollouts used to calculate the value targets, we augment the predicted rewards $\hat{r}_t$ with our potential-based shaping term:
    \begin{equation}
        \tilde{r}_t = \hat{r}_t + \gamma \Phi(\hat{\mathbf{s}}_{t+1}) - \Phi(\hat{\mathbf{s}}_t),
    \end{equation}
    where $\Phi(\cdot)$ is the bootstrapped potential. This ensures that the Critic learns a value function that incorporates global directional guidance, subsequently guiding the Actor, while the World Model remains an accurate simulator of the environment.
\end{itemize}

\paragraph{Results and Analysis.}
We evaluated this integration on two sparse-reward tasks from the \textit{Meta-World} benchmark: \texttt{Coffee-Push} and \texttt{Window-Open}. As shown in Figure \ref{fig:dreamer}, Dreamerv3 combined with SLOPE (Orange curves) significantly outperforms the vanilla Dreamerv3 (Blue curves). These results confirm that the benefits of potential landscape construction are algorithm-agnostic and can also effectively enhance policy optimization in decoupled MBRL architectures.

\begin{figure}[t!]
    \centering
    \includegraphics[width=0.7\linewidth]{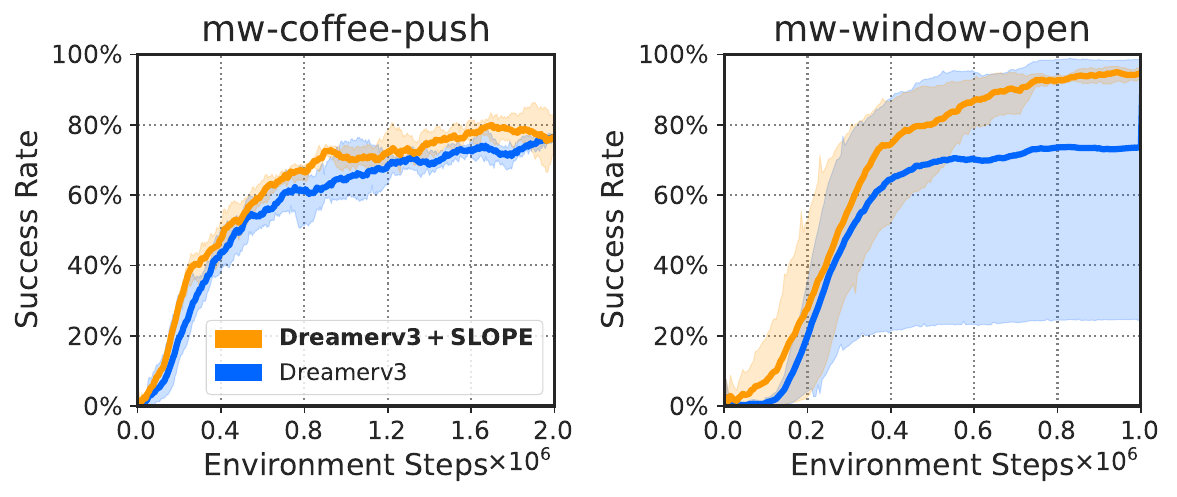}
    \caption{Performance comparison between Dreamerv3 + SLOPE and vanilla Dreamerv3 on \textit{Meta-World} sparse reward tasks. By injecting the potential-based shaping signal into the critic's target computation during latent imagination, SLOPE significantly accelerates learning and improves final performance compared to the strong baseline.}
    \vskip -0.1in
    \label{fig:dreamer}
\end{figure}

\subsection{Integration with LS-Imagine}
\label{sec:ls_imagine}
\begin{wrapfigure}[13]{r}{0.35\textwidth}
    \vskip -0.4in
    \centering
    \includegraphics[width=\linewidth]{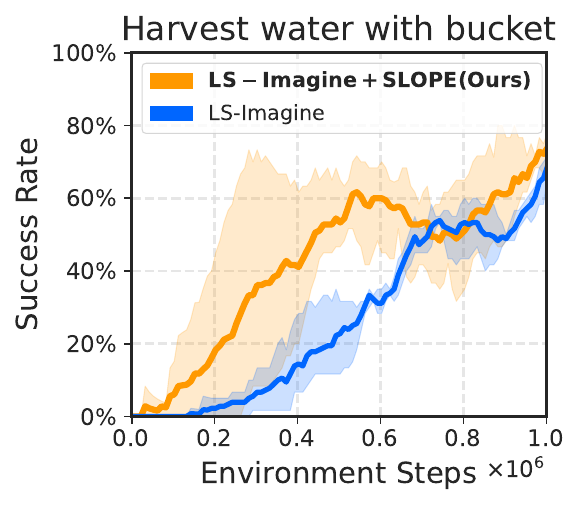}
    \caption{Success-rate of LS-Imagine with and without SLOPE on \textit{Harvest water with bucket}.}
    \label{fig:ls_imagine}
    \vspace{-8pt}
\end{wrapfigure}

To evaluate SLOPE beyond continuous-control robotics, we integrate it with LS-Imagine \cite{ls_imagine} and evaluate on its native MineDojo task, \textit{Harvest water with bucket}. We compare the original LS-Imagine with LS-Imagine augmented by SLOPE over one million environment steps. This experiment examines whether SLOPE provides complementary benefits within an existing framework for long-horizon exploration.

As shown in Fig.~\ref{fig:ls_imagine}, incorporating SLOPE accelerates learning, particularly during the first half of training. The augmented agent reaches higher success rates substantially earlier, whereas the original LS-Imagine requires more interactions to achieve comparable performance. The performance gap narrows later in training, but adding SLOPE still yields a higher average final success rate. The learning curves therefore indicate a substantial benefit in learning efficiency, accompanied by a modest improvement in final performance.

These results suggest that SLOPE's value-based shaping complements LS-Imagine's long-horizon imagination mechanism. By providing additional task-directed guidance through the learned reward model, SLOPE improves learning on this open-world task, extending its demonstrated applicability beyond robotic manipulation.

\subsection{Long-Horizon Hard Exploration Tasks}
\label{sec:long_horizon_exploration}

We evaluate SLOPE in two complementary long-horizon settings: high-dimensional continuous control in HumanoidBench and
discrete-action exploration in Montezuma's Revenge. All agents learn from scratch without demonstrations or pre-collected successful trajectories.

\paragraph{HumanoidBench.}
We evaluate on \texttt{h1hand-hurdle} and \texttt{h1hand-maze} from HumanoidBench \cite{humanoidbench}, using the Unitree H1 humanoid with dexterous hands. Both tasks have a 1,000-step episode horizon, and training runs for two million interaction steps. The benchmark authors identify hurdle as a hard-exploration problem, where agents learn forward locomotion but fail to discover the jumping behavior needed to overcome obstacles. As shown in Fig.~\ref{fig:long_horizon_exploration}, SLOPE achieves substantially higher returns than TD-MPC2 on hurdle. Its return continues to improve throughout training, whereas the baseline makes limited progress. On maze, SLOPE also improves returns earlier and achieves higher final performance, although the final gap is smaller. The larger variability on hurdle indicates that learning remains challenging despite the substantial improvement in average return.

\paragraph{Montezuma's Revenge.}
To examine a canonical hard-exploration domain, we integrate SLOPE with DreamerV3 to accommodate the discrete action space and train for ten million environment steps. Figure~\ref{fig:long_horizon_exploration} shows that DreamerV3+SLOPE achieves higher average scores during most of training and finishes above vanilla DreamerV3.

Together, these experiments show that SLOPE's benefits extend beyond demonstration-supported manipulation. It improves learning from scratch across continuous and discrete action spaces with different MBRL backbones. The gains vary by task, with a pronounced improvement on hurdle and more moderate improvements on maze and Montezuma's Revenge. These results support its applicability to the evaluated long-horizon exploration settings.

\begin{figure}[h!]
    \centering
    \includegraphics[width=0.95\linewidth]{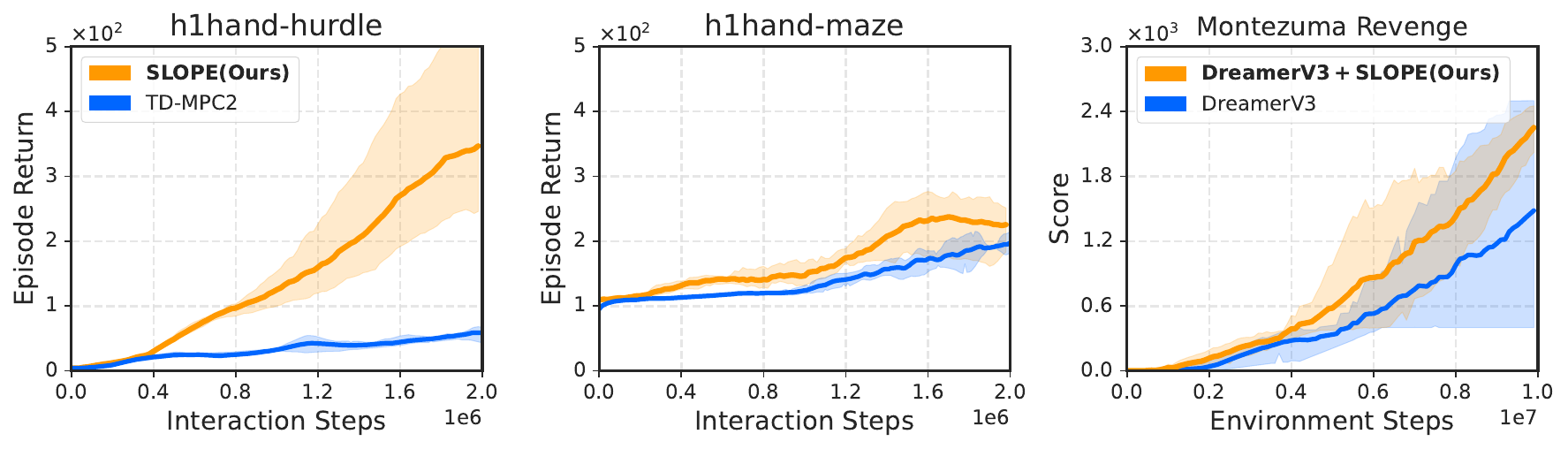}
    \caption{Learning curves on long-horizon exploration tasks. We evaluate SLOPE with DreamerV3 on Montezuma's Revenge
    over ten million environment steps and with TD-MPC2 on HumanoidBench over two million interaction steps. Both HumanoidBench tasks have a 1,000-step episode horizon.}
    \label{fig:long_horizon_exploration}
    \vskip -0.1in
\end{figure}
\clearpage

\subsection{Standalone Performance in Zero-Demonstration Settings}
To provide a more comprehensive evaluation of SLOPE's foundational capabilities, we assess its performance in the absence of expert demonstrations within fully sparse environments. As shown in Fig.~\ref{fig:wo_demos}, SLOPE effectively learns from scratch and achieves task success without any prior guidance. In contrast, the baseline TD-MPC2 completely fails to overcome the severe exploration bottleneck, consistently yielding near-zero returns.

Despite this standalone effectiveness, we emphasize the highly complementary relationship between our shaping framework and expert demonstrations. While demonstrations provide critical high-value initial states, SLOPE structurally propagates these localized signals into a dense, informative landscape. This synergy maximizes the utility of limited data, significantly enhancing sample efficiency for real-world applications where unguided exploration is prohibitively expensive.

\begin{figure}[h!]
    \centering
    \includegraphics[width=0.95\linewidth]{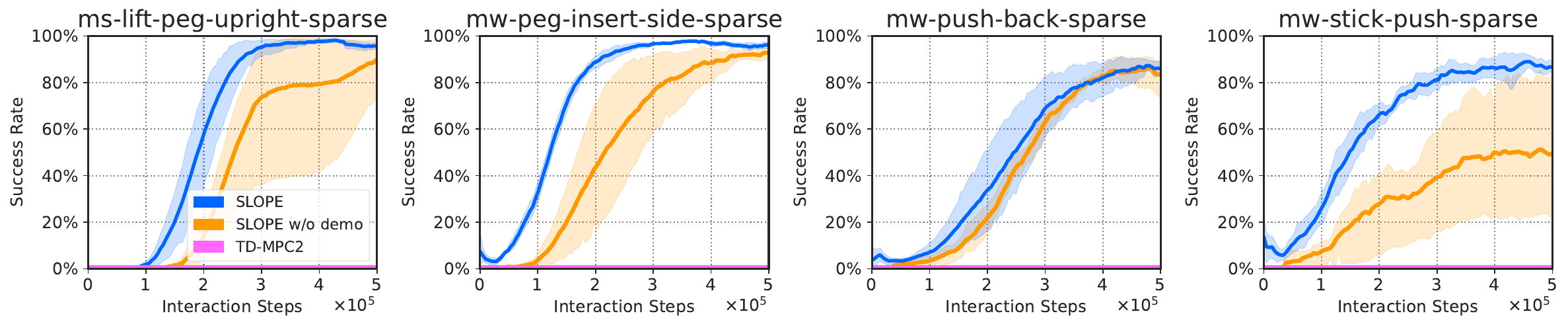}
    \caption{Performance of SLOPE without demonstrations in fully sparse reward settings.}
    \label{fig:wo_demos}
\end{figure}

\subsection{Performance on Semi-Sparse Tasks}
We also evaluate the performance of SLOPE on semi-sparse tasks, where sparse-reward environments are augmented with handcrafted stage rewards through manual task decomposition. As shown in Fig.~\ref{fig:semi_rs} (Left), SLOPE achieves significantly better performance in the semi-sparse setting compared to the fully sparse rewards. This improvement demonstrates that while SLOPE is not tailored for semi-sparse tasks, it still benefits from the structured rewards and achieves competitive results.

\subsection{Comparison with Other Reward Shaping Methods}
We compare SLOPE with two alternative reward shaping methods: similarity-based and entropy-based potentials, both designed to encourage exploration. The similarity-based potential $\Phi _{\text{sim}}\left( \mathbf{s} \right) =\lVert \mathbf{z'}-\texttt{stop\_grad}\left( h_{\theta}\left( \mathbf{s'} \right) \right) \rVert _{2}^{2}$ drives the agent toward novel states, while the entropy-based potential $\Phi_{\text{entroy}}(\mathbf{s})=\mathcal{H}\left( \pi(\cdot|\mathbf{s}) \right)$ promotes diverse behavior through higher policy entropy.

As shown in Figure~\ref{fig:semi_rs} (Right), both methods achieve reasonable performance on relatively simple tasks, such as \texttt{ms-poke-cube-sparse}, where exploratory behaviors are easier to reward. However, as the task complexity increases, the performance of both similarity-based and entropy-based shaping methods degrades significantly. In contrast, SLOPE consistently maintains strong performance across these harder tasks. This suggests that in complex sparse-reward environments, general exploration incentives are insufficient, and agents benefit more from reward signals that are explicitly aligned with task progression.

\begin{figure}[h!]
    \centering
    \begin{subfigure}[t]{0.48\linewidth}
        \centering
        \includegraphics[width=\linewidth]{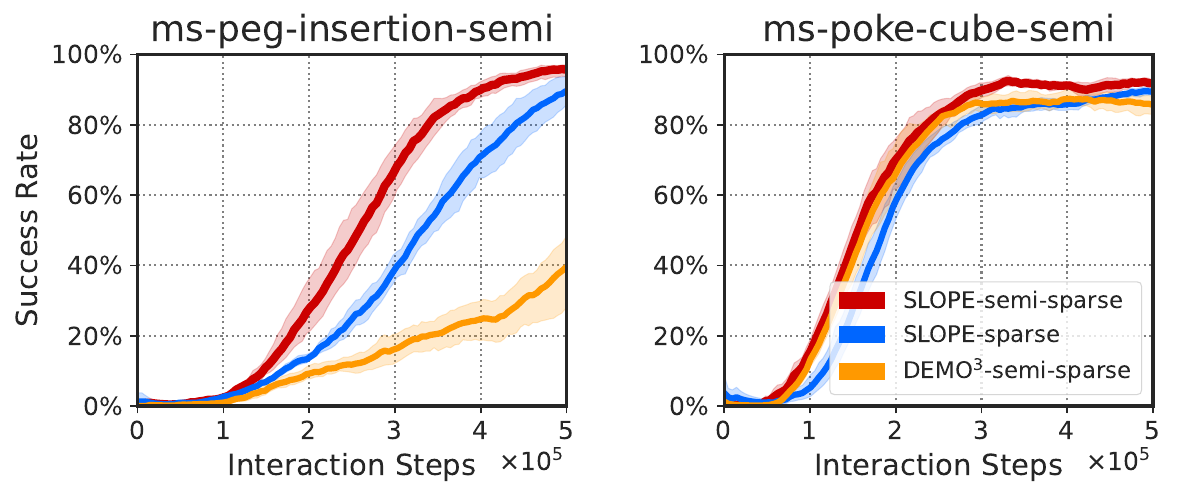}
    \end{subfigure}
    \hfill
    \begin{subfigure}[t]{0.48\linewidth}
        \centering
        \includegraphics[width=\linewidth]{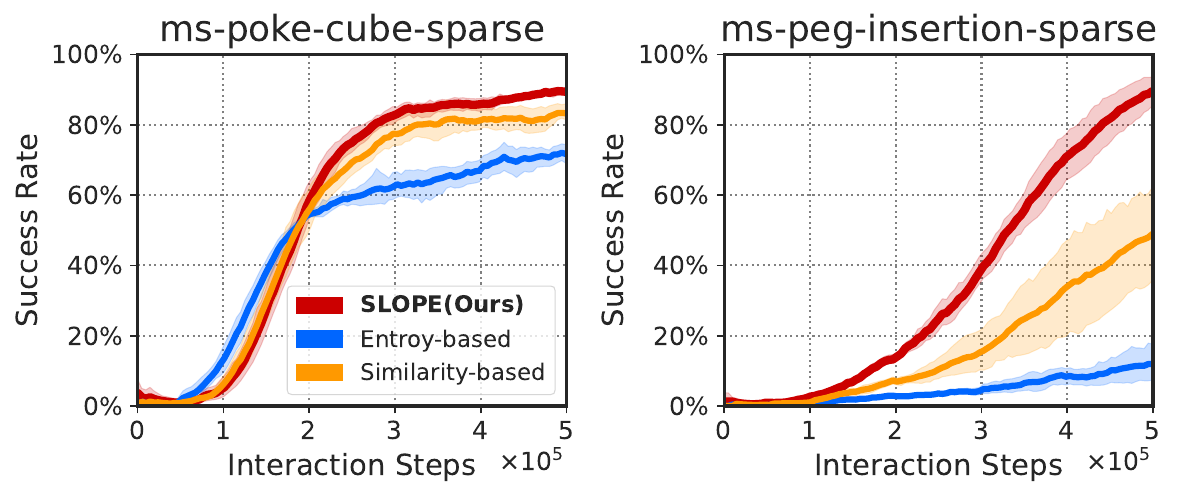}
    \end{subfigure}
   \caption{\textbf{Left}: Performance evaluation of SLOPE in semi-sparse reward settings. \textbf{Right}: Comparison of different reward shaping methods.}
   \vskip -0.15in
    \label{fig:semi_rs}
\end{figure}

\clearpage

\subsection{Comparison with Reward Smoothing: DreamerSmooth}
To provide a more comprehensive evaluation of SLOPE, we compare it against DreamSmooth \cite{dreamsmooth}, a representative MBRL algorithm designed to mitigate reward prediction bottlenecks. Its methodology utilizes temporal relaxation to predict a temporally-smoothed reward instead of the exact scalar at a specific timestep, thereby addressing prediction ambiguity.

\paragraph{Core Implementation.}
Unlike SLOPE, which is designed for purely sparse-reward environments, DreamSmooth is typically applied to tasks with existing dense reward structures augmented by high-magnitude sparse signals at key stage points. We reimplemented this approach by applying a Gaussian smoothing function $f$ to rewards $r$ during replay buffer storage:
\begin{equation}
    \tilde{r}_t \leftarrow f(r_{t-L:t+L}) = \sum_{i=-L}^{L} f_i \cdot r_{t+i} \quad \text{s.t.} \quad \sum f_i = 1
\end{equation}
Our implementation adheres to the following core principles:
\begin{itemize}
    \item \textbf{Temporal Relaxation}: By diffusing the sparse reward signal over a temporal horizon $L$, the model can predict rewards several timesteps earlier or later without incurring large losses.

    \item \textbf{Sum Preservation}: The smoothing function is normalized to ensure $\sum \tilde{r} = \sum r$, preserving policy optimality under the original reward objective.

    \item \textbf{Gaussian Kernel}: We utilize a Gaussian kernel ($\sigma=8$) to provide the planner with a smooth, directional gradient toward the goal, transforming the sparse rewards into dense.
\end{itemize}

\begin{figure}[h!]
    \centering
    \includegraphics[width=0.95\linewidth]{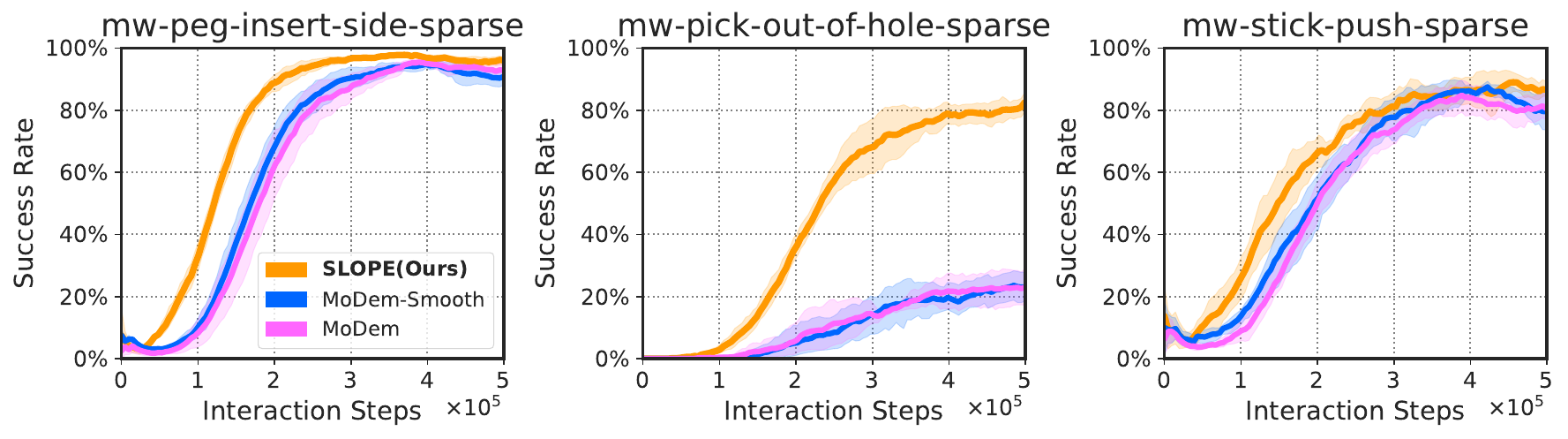}
    \caption{Comparative analysis between SLOPE and Reward Smoothing.}
    \label{fig:smooth}
\end{figure}

\paragraph{Results and Analysis.}
As illustrated in Fig.~\ref{fig:smooth}, SLOPE consistently outperforms the reward smoothing baseline (MoDem-Smooth) across multiple \textit{Meta-World} tasks, achieving both faster convergence and higher asymptotic success rates. Notably, the integration of temporal smoothing into MoDem yields no significant performance gain over the original sparse baseline. While temporal relaxation in DreamSmooth eases the reward prediction bottleneck by diffusing signals over a temporal window $L$, it remains fundamentally limited by the scarcity of the underlying scalar signal. In fully sparse-reward environments, the absence of dense supervision or stage-wise priors means that temporal diffusion merely attenuates the infrequent success signals rather than constructing an informative gradient. Consequently, reward smoothing fails to bridge the exploration gap in challenging tasks like \texttt{pick-out-of-hole}. In contrast, SLOPE's potential-based reward provides precise directional guidance by leveraging structured potential landscapes, offering a more robust planning signal than simple Gaussian diffusion ($\sigma=8$) for complex robotic manipulation.

\clearpage
\subsection{Effect of the Quantile Coefficient}
\label{sec:additional_tau_ablation}

To examine the effect of mild optimism across environments, we compare $\tau=0.55$ with the symmetric objective ($\tau=0.5$) on three sparse-reward tasks: Pick Out of Hole, Pick Place Wall, and Peg Insertion. 

\begin{figure}[t!]
    \centering
    \includegraphics[width=0.95\linewidth]{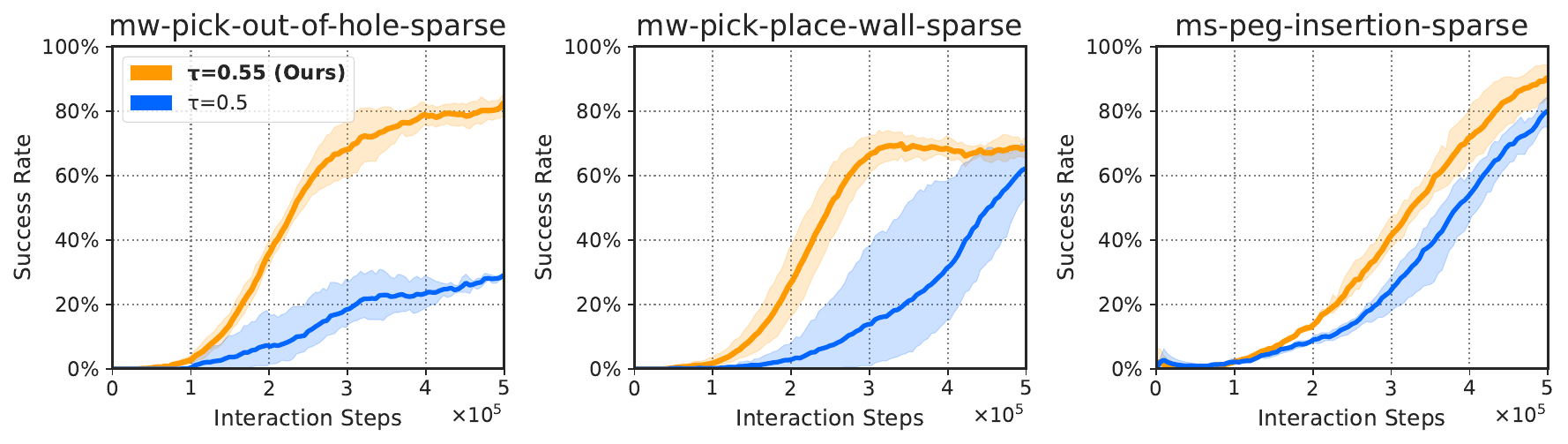}
    \caption{Effect of the quantile coefficient on three sparse-reward tasks. Mild optimistic weighting ($\tau=0.55$, orange) improves learning speed and average final success rates compared with the symmetric objective ($\tau=0.5$, blue). The largest final-performance gain occurs on Pick Out of Hole.}
    \label{fig:ablation_tau}
    \vskip -0.1in
\end{figure}

As shown in Fig.~\ref{fig:ablation_tau}, $\tau=0.55$ improves learning on all three evaluated tasks, although the magnitude of the benefit varies. On Pick Out of Hole, it produces a much steeper improvement during early training and raises average final success. On Pick Place Wall, the optimistic objective reaches high success rates substantially earlier, while the symmetric objective partially catches up near the end of training. The final improvement is therefore smaller on this task. On Peg Insertion, $\tau=0.55$ also accelerates learning and maintains higher success rates during the later training stages.

\subsection{Ablating Potential Shaping and Optimistic Learning}
\label{sec:shaping_optimism_ablation}

To separate the contributions of potential shaping and optimistic learning, we compare four variants: MoDem, MoDem with potential shaping only, MoDem with optimistic learning only, and SLOPE, which combines both components. We use Pick Place and Peg Insertion, two sparse-reward ManiSkill3 tasks on which MoDem achieves relatively low performance. All variants share the demonstrations, TD-MPC2 backbone, interaction budget, and evaluation protocol.

\begin{figure}[h!]
    \centering
    \includegraphics[width=0.7\linewidth]{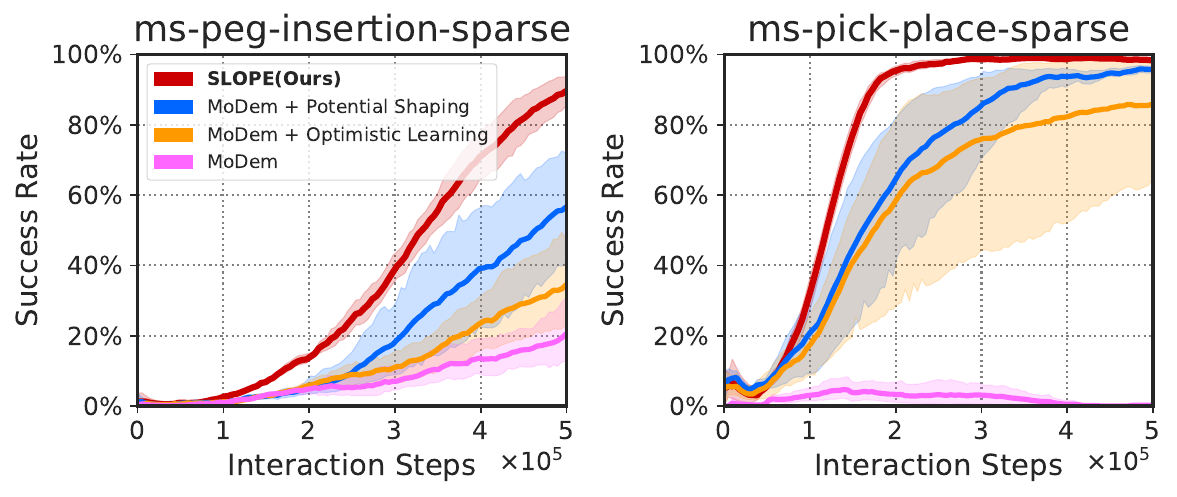}
    \caption{Separate ablations of potential shaping and optimistic learning on Peg Insertion and Pick Place. Potential shaping alone substantially improves over MoDem, while adding optimistic learning further improves final performance on Peg Insertion and accelerates learning on Pick Place.}
    \label{fig:shaping_optimism_ablation}
    \vskip -0.1in
\end{figure}

Figure~\ref{fig:shaping_optimism_ablation} shows that potential shaping alone substantially improves over MoDem on both tasks and achieves higher final performance than optimistic learning alone. This indicates that shaping accounts for a substantial portion of SLOPE's performance gains.

Optimistic learning provides complementary, task-dependent benefits. On Peg Insertion, removing optimism slows learning and leaves a substantial gap in final performance, showing that shaping alone does not fully explain SLOPE's effectiveness. On Pick Place, shaping alone eventually approaches SLOPE's final success rate, but requires considerably more interactions to reach high performance. Optimism therefore primarily accelerates learning on this task.

These results indicate that potential shaping accounts for much of the performance improvement on the two evaluated tasks. Optimistic learning complements this guidance by placing greater weight on high-return targets, with benefits reflected in higher final performance on Peg Insertion and faster learning on Pick Place. Combining both components achieves the best overall performance.

\subsection{Computational Resources}
To evaluate the efficiency of our approach, we conducted a comparative experiment using the \textit{ManiSkill3} benchmark to assess the computational resource requirements of SLOPE against various baseline algorithms. To ensure a fair comparison of resource usage, these benchmarking experiments were all conducted on a high-performance computing server equipped with dual Intel Xeon Gold 6240 CPUs (totaling 36 cores and 72 threads, 2.60 GHz) and 10 NVIDIA GeForce RTX 2080 Ti GPUs.

\begin{table}[h]
\centering
\vskip -0.1in
\caption{Comparison of training resources for different algorithms. The best score is marked with \textbf{bold}. The direction of the arrow indicates whether the desired metric is large or small.}
\resizebox{\linewidth}{!}{
\begin{tabular}{lcccc}
\midrule
\textbf{Method}      & \textbf{Training time} $\downarrow$ & \textbf{CPU memory}  $\downarrow$& \textbf{Learnable parameters} $\downarrow$ & \textbf{Average success rate $\uparrow$} \\ 
\midrule
\rowcolor{graybg}
TD-MPC2     & \textbf{7.9h}          & \textbf{29.53GB}    & \textbf{5.6M}                 & 0\%                       \\
\rowcolor{graybg}
MoDem       & 9.7h          & 30.48GB    & \textbf{5.6M}                 & 55.0\%                       \\
\rowcolor{graybg}
DEMO3-semi-sparse       & 11.8h         & 30.48GB    & 5.7M                 & 74.0\%                     \\ 
\rowcolor{graybg}
DEMO3-sparse       & 11.8h         & 30.48GB    & 5.7M                 & 54.8\%                     \\ 
\midrule 
\rowcolor{bluebg}
\textbf{SLOPE(Ours)} & 9.9h          & 40.37GB    & \textbf{5.6M}                 & \textbf{92.0\%}                     \\ \midrule
\end{tabular}
}
\vskip -0.1in
\label{table:comput}
\end{table}

As illustrated in Table \ref{table:comput}, although SLOPE requires a higher allocation of CPU memory (40.37 GB) compared to the baseline, this additional resource consumption translates into a substantial gain in performance. Specifically, SLOPE achieves a remarkable average success rate of \textbf{92.0\%}, surpassing the strongest baseline (DEMO$^3$-semi-sparse) by a margin of 18.0\%. The increased memory footprint stems from augmenting the demonstration buffer with successful trajectories encountered during training, a strategy designed to mitigate data imbalance in sparse-reward tasks. Furthermore, despite the increased memory footprint, our method maintains high computational efficiency with a training time of 9.9 hours, which is comparable to MoDem and notably faster than DEMO$^3$ (11.8 hours).

\clearpage

\subsection{Visualization of Real-World Policy Rollouts}
In this section, we provide the visualization of the learned policies deployed on the real robot. Figure \ref{fig:task_process} illustrates the frame-by-frame execution sequences of our method, SLOPE, across the three sparse-reward manipulation tasks. As shown in the snapshots, the agent effectively controls the robotic arm to approach the target objects and perform precise interactions, demonstrating the robustness of the learned behaviors in real-world scenarios.

\begin{figure}[h!]
    \centering
    \includegraphics[width=1.0\linewidth]{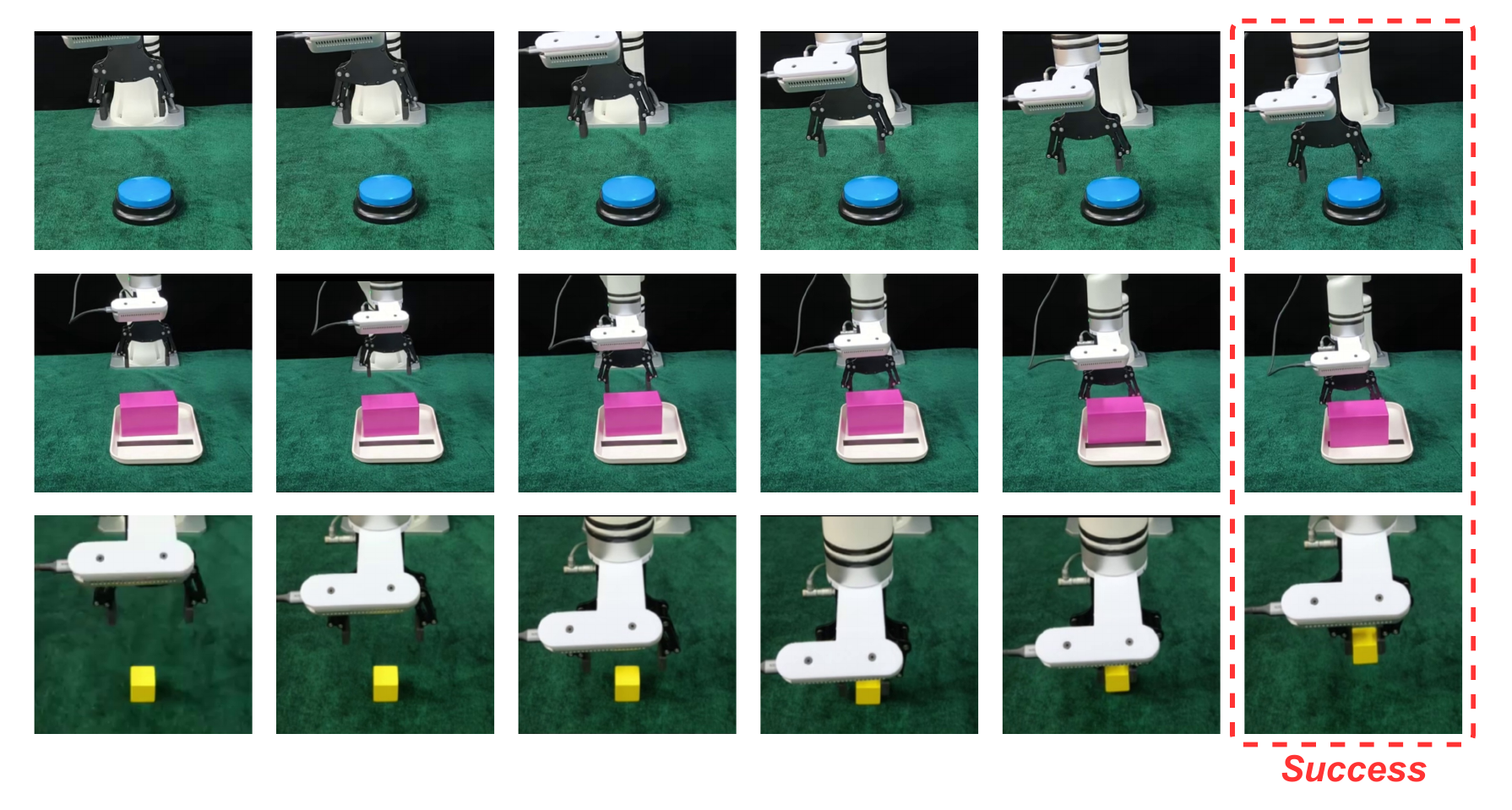}
    \caption{Visualization of successful policy rollouts in the real world. The figure displays chronological snapshots (left to right) of the agent executing the \textbf{Press Button} (top row), \textbf{Push Cube} (middle row), and \textbf{Grasp Cube} (bottom row) tasks. The final column, marked with a red dashed box, highlights the successful completion state for each sparse-reward task.}
    \label{fig:task_process}
\end{figure}

\clearpage
\section{Algorithm}
The pseudocode of our proposed method is presented in Algorithm \ref{alg:slope}.

Our implementation of SLOPE is designed to be backbone-agnostic while carefully adapting to the specific architectural constraints of the underlying algorithm. When integrating with TD-MPC2, which utilizes an ensemble of Q-networks, we adopt an asymmetric aggregation strategy to calculate the shaping term $\gamma \Phi(s') - \Phi(s)$. Specifically, we compute the current potential $\Phi(s)$ using the ensemble average, aligning with TD-MPC2's standard practice for low-variance value estimation. However, for the next state potential $\Phi(s')$, we employ the ensemble minimum. As PBRS is particularly sensitive to the next-state term, where any overestimation in $\Phi(s')$ directly translates into an artificially inflated reward signal, utilizing the minimum operator effectively mitigates overestimation bias. This ensures that the shaping signal reflects reliable potential gains rather than epistemic noise, while maintaining a stable baseline through the averaged current potential.

This strategy generalizes to architectures that employ standard explicit target networks, such as Dreamerv3 or our real-world robotic control frameworks. In these settings, SLOPE computes $\Phi(s)$ using the online network and $\Phi(s')$ using the slow-moving target network, strictly adhering to the stability mechanisms of the respective backbones. This flexibility ensures that SLOPE provides consistent, risk-averse guidance across diverse MBRL implementations, whether they rely on ensemble statistics or frozen target parameters.

\begin{algorithm}[h]
\caption{SLOPE: Shaping Potential Landscapes for MBRL}
\label{alg:slope}
\textbf{Input}: $\mathcal{D}^{\mathrm{demo}}$: demonstration dataset, $\mathcal{B}$: replay buffer (initially empty)\\
\textbf{Parameter}: $\theta$: model parameters, $\pi_{\theta}$: policy, $\Pi_{\theta}$: planner, $\tau$: quantile estimation coefficient, $\eta$: reward shaping weight\\
\textbf{Output}: $\Pi_{\theta}$: planner, $\pi_{\theta}$: trained policy
\begin{algorithmic}[1] 
\STATE \textit{\textcolor{blue}{// Phase 1: Policy Pretraining}}
\FOR{each policy update}
  \STATE Sample $(s_t, a_t) \sim \mathcal{D}^{\mathrm{demo}}$
  \STATE Update $h_{\theta}, \pi_{\theta}$ using BC loss $\mathcal{L}_{\mathrm{BC}}(\theta)$
\ENDFOR
\STATE \textit{\textcolor{blue}{// Phase 2: Model Seeding}}
\FOR{each seeding rollout}
  \STATE Collect rollout $\mathcal{T} \leftarrow \{s_t, a_t, r_t, s_{t+1}\}_{0:T}$ using $a_t \sim \pi_{\theta}(s_t)$
  \STATE Add trajectory to buffer: $\mathcal{D}^{\mathrm{seed}} \leftarrow \mathcal{D}^{\mathrm{seed}} \cup \mathcal{T}$
\ENDFOR
\FOR{each model update}
  \STATE Sample $(s_t, a_t, r_t, s_{t+1})_{0:T}$ from $\mathcal{D}^{\mathrm{demo}} \cup \mathcal{D}^{\mathrm{seed}}$
  \STATE Update $h_{\theta}, d_{\theta}, R_{\theta}, Q_{\theta}, \pi_{\theta}$ using TD-MPC loss $\mathcal{L}_{\mathrm{TD-MPC}}(\theta)$ 
\ENDFOR
\STATE \textit{\textcolor{blue}{// Phase 3: Interactive Learning with SLOPE Shaping}}
\STATE Initialize buffer: $\mathcal{B} \coloneqq \mathcal{D}^{\mathrm{seed}}$
\WHILE{interaction budget not exhausted}
  \STATE Collect rollout $\mathcal{T} \leftarrow \{s_t, a_t, r_t, s_{t+1}\}_{0:T}$ using $a_t \sim \Pi_{\theta}(s_t)$
  \STATE Add trajectory to buffer: $\mathcal{B} \leftarrow \mathcal{B} \cup \mathcal{T}$
  \STATE Sample $(s_t, a_t, r_t, s_{t+1})_{t:t+H}$ from $\mathcal{D}^{\mathrm{demo}} \cup \mathcal{B}$
  \STATE Compute shaped reward: $\tilde{r}_t = r_t + \Delta r_t$, where $\Delta r_t = \gamma \mathbb{E}_{s'}[\Phi(s')] - \Phi(s_t)$
  \STATE Update $h_{\theta}, d_{\theta}, R_{\theta}, Q_{\theta}, \pi_{\theta}$ using TD-MPC loss $\mathcal{L}_{\mathrm{TD-MPC}}(\theta)$ 
\ENDWHILE
\STATE \textbf{return} $\Pi_{\theta}$
\end{algorithmic}
\end{algorithm}

\clearpage

\newpage
\section*{NeurIPS Paper Checklist}

\begin{enumerate}

\item {\bf Claims}
    \item[] Question: Do the main claims made in the abstract and introduction accurately reflect the paper's contributions and scope?
    \item[] Answer: \answerYes{} 
    \item[] Justification: See Abstract and Introduction.

\item {\bf Limitations}
    \item[] Question: Does the paper discuss the limitations of the work performed by the authors?
    \item[] Answer: \answerYes{} 
    \item[] Justification: See Conclusion.

\item {\bf Theory assumptions and proofs}
    \item[] Question: For each theoretical result, does the paper provide the full set of assumptions and a complete (and correct) proof?
    \item[] Answer: \answerYes{} 
    \item[] Justification: See detailed proofs in Appendix.

    \item {\bf Experimental result reproducibility}
    \item[] Question: Does the paper fully disclose all the information needed to reproduce the main experimental results of the paper to the extent that it affects the main claims and/or conclusions of the paper (regardless of whether the code and data are provided or not)?
    \item[] Answer: \answerYes{} 
    \item[] Justification: The implementation details are provided in Appendix and the code is linked anonymously in the Abstract.

\item {\bf Open access to data and code}
    \item[] Question: Does the paper provide open access to the data and code, with sufficient instructions to faithfully reproduce the main experimental results, as described in supplemental material?
    \item[] Answer: \answerYes{} 
    \item[] Justification: The code is linked anonymously in the Abstract.

\item {\bf Experimental setting/details}
    \item[] Question: Does the paper specify all the training and test details (e.g., data splits, hyperparameters, how they were chosen, type of optimizer) necessary to understand the results?
    \item[] Answer: \answerYes{} 
    \item[] Justification: See implementation details in Appendix.

\item {\bf Experiment statistical significance}
    \item[] Question: Does the paper report error bars suitably and correctly defined or other appropriate information about the statistical significance of the experiments?
    \item[] Answer: \answerYes{} 
    \item[] Justification: The curves and shaded areas in our paper represent the mean and 95\% confidence intervals (CIs) over 5 independent runs.

\item {\bf Experiments compute resources}
    \item[] Question: For each experiment, does the paper provide sufficient information on the computer resources (type of compute workers, memory, time of execution) needed to reproduce the experiments?
    \item[] Answer: \answerYes{} 
    \item[] Justification: See Section Computational Resources in Appendix.
    
\item {\bf Code of ethics}
    \item[] Question: Does the research conducted in the paper conform, in every respect, with the NeurIPS Code of Ethics \url{https://neurips.cc/public/EthicsGuidelines}?
    \item[] Answer: \answerYes{} 
    \item[] Justification: We have reviewed the NeurIPS Code of Ethics and ensure that our paper complies with them.

\item {\bf Broader impacts}
    \item[] Question: Does the paper discuss both potential positive societal impacts and negative societal impacts of the work performed?
    \item[] Answer: \answerNA{} 
    \item[] Justification: This paper focuses on fundamental algorithmic improvements to MBRL in simulated environments. As foundational research without a direct path to specific, deployable applications, it does not present immediate positive or negative societal impacts.
    
\item {\bf Safeguards}
    \item[] Question: Does the paper describe safeguards that have been put in place for responsible release of data or models that have a high risk for misuse (e.g., pre-trained language models, image generators, or scraped datasets)?
    \item[] Answer: \answerNA{} 
    \item[] Justification: This paper poses no such risks.

\item {\bf Licenses for existing assets}
    \item[] Question: Are the creators or original owners of assets (e.g., code, data, models), used in the paper, properly credited and are the license and terms of use explicitly mentioned and properly respected?
    \item[] Answer: \answerYes{} 
    \item[] Justification: We have cite the original paper that produced the code package or dataset correctly.

\item {\bf New assets}
    \item[] Question: Are new assets introduced in the paper well documented and is the documentation provided alongside the assets?
    \item[] Answer: \answerYes{} 
    \item[] Justification: The demonstrations used in our paper in released in the anonymized URL.

\item {\bf Crowdsourcing and research with human subjects}
    \item[] Question: For crowdsourcing experiments and research with human subjects, does the paper include the full text of instructions given to participants and screenshots, if applicable, as well as details about compensation (if any)? 
    \item[] Answer: \answerNA{} 
    \item[] Justification: This paper does not involve crowdsourcing nor research with human subjects.

\item {\bf Institutional review board (IRB) approvals or equivalent for research with human subjects}
    \item[] Question: Does the paper describe potential risks incurred by study participants, whether such risks were disclosed to the subjects, and whether Institutional Review Board (IRB) approvals (or an equivalent approval/review based on the requirements of your country or institution) were obtained?
    \item[] Answer: \answerNA{} 
    \item[] Justification: This paper does not involve crowdsourcing nor research with human subjects.

\item {\bf Declaration of LLM usage}
    \item[] Question: Does the paper describe the usage of LLMs if it is an important, original, or non-standard component of the core methods in this research? Note that if the LLM is used only for writing, editing, or formatting purposes and does \emph{not} impact the core methodology, scientific rigor, or originality of the research, declaration is not required.
    \item[] Answer: \answerNA{} 
    \item[] Justification: The LLM is used only for editing and formatting.

\end{enumerate}

\end{document}